\documentclass[final, 12pt]{clear2025} %

\title[Counterfactual Influence in MDPs]{Counterfactual Influence in Markov Decision Processes}
\usepackage{times}
\usepackage[utf8]{inputenc} %
\usepackage[T1]{fontenc}    %
\usepackage{hyperref}       %
\usepackage{url}            %
\usepackage{booktabs}       %
\usepackage{amsfonts}       %
\usepackage{nicefrac}       %
\usepackage{microtype}      %
\usepackage{xcolor}         %
\usepackage{pgfplots}
\pgfplotsset{compat=1.9, width=\columnwidth, height=6cm}
\usepgfplotslibrary{fillbetween}
\usepackage{tikz}
\usepackage{tablefootnote}
\usepackage{amsmath}
\usepackage{circuitikz}
\usepackage{amssymb}
\usepackage{bbm}
\usepackage{mathtools}
\usepackage{algorithm}
\usepackage{caption}
\usepackage{algorithmic}
\usetikzlibrary{positioning}
\usepackage{wrapfig}
\usepackage{xcolor}

\usepackage{svg}

\newcommand{\jl}[1]{{\color{black}#1}}

\newcommand{\mk}[1]{{\color{black}#1}}
\DeclareMathOperator*{\argmax}{arg\,max}

\newcommand{\scm}{\mathcal{C}}
\newcommand{\ndovars}{\mathbf{V}}
\newcommand{\exovars}{\mathbf{U}}

\newcommand{\mdp}{\mathcal{M}}
\newcommand{\states}{\mathcal{S}}
\newcommand{\actions}{\mathcal{A}}

\newcommand{\probs}[1][]{%
\ifthenelse{\equal{#1}{}}{P_{\mdp}}{P_{\mdp(#1)}}%
}

\definecolor{highlightcol}{HTML}{fcf9d1}
\definecolor{highlightcol1}{HTML}{d1c964}

\tikzset{
  main node/.style={
    circle,
    draw,
    thick,
  },
}

\definecolor{darkgreen}{rgb}{0.0, 0.5, 0.0}

\clearauthor{\Name{Milad Kazemi}\thanks{Equal contribution. }\Email{milad.kazemi@kcl.ac.uk}\\
 \Name{Jessica Lally}\footnotemark[1] \Email{jessica.lally@kcl.ac.uk}\\
 \Name{Ekaterina Tishchenko} \Email{ekaterina.tishchenko@kcl.ac.uk}\\
 \Name{Hana Chockler} \Email{hana.chockler@kcl.ac.uk}\\
 \Name{Nicola Paoletti} \Email{nicola.paoletti@kcl.ac.uk}\\
 \addr Department of Informatics, King's College London, London, UK}

\begin{document}

\maketitle

\begin{abstract}%

Our work addresses a fundamental problem in the context of counterfactual inference for Markov Decision Processes (MDPs). Given an MDP path $\tau$, counterfactual inference allows us to derive counterfactual paths $\tau'$ describing \textit{what-if} versions of $\tau$ obtained under different action sequences than those observed in $\tau$. However, as the counterfactual states and actions deviate from the observed ones over time, \textit{the observation $\tau$ may no longer influence the counterfactual world}, meaning that the analysis is no longer tailored to the individual observation, resulting in interventional outcomes rather than counterfactual ones. This issue specifically affects the popular Gumbel-max structural causal model used for MDP counterfactuals, and yet, it has remained overlooked until now. In this work, we introduce a formal characterisation of influence based on comparing counterfactual and interventional distributions. We devise an algorithm to construct counterfactual models that automatically satisfy influence constraints. Leveraging such models, we derive counterfactual policies that are not just optimal for a given reward structure but also remain tailored to the observed path. Even though there is an unavoidable trade-off between policy optimality and strength of influence constraints, our experiments demonstrate that it is possible to derive (near-)optimal policies while remaining under the influence of the observation.
\end{abstract}


\section{Introduction}

Counterfactual inference allows us to reason hypothetically about what effect changing an action or condition in the past would have on a given observation, e.g., 
\textit{``What would the patient's condition be, had we used treatment Y instead of treatment X?''}. 
This allows us to evaluate and optimise sequences of actions, by identifying how the outcome could have been improved by changing some action(s) along the observed path. This optimised counterfactual path can be seen as a \textit{counterfactual explanation} for how the decision-making process could be improved. Markov Decision Processes (MDPs) \jl{are particularly useful for modelling real-world decision-making processes under uncertainty}, so counterfactual inference can be used to generate counterfactual explanations for improving a given policy. This research area aligns with the increasing use of machine learning in healthcare, specifically in supporting clinical decision-making with the use of AI models~\citep{verma2020counterfactual,guidotti2022counterfactual,tsirtsis2021counterfactual}. In particular, we focus on discrete-state MDPs, a fundamental computation model used in e.g., clinical decision-making~\citep{BENNETT20139, SHIFRIN2020101917, patrick2011markov}, games~\citep{hafner2023mastering}, and planning tasks~\citep{natarajan2022planning}.

However, one previously neglected issue with performing counterfactual inference in sequential decision processes (which is not an issue in single-step scenarios) is that, after some number of steps, the observation may no longer inform the counterfactual transition probabilities of the counterfactual path. Consider an observed path $\tau=(s_1,a_1,s_2,a_2,s_3)$. Formally, each $s_{t+1}$ arises from a structural causal model (SCM), $s_{t+1} \sim f(s_t,a_t,\mathbf{U})$, where $\mathbf{U}$ are the (prior) random exogenous factors and $f$ is a deterministic function. We can apply counterfactual inference to determine, for example, which states we would have reached if we had instead performed action $a_1'\neq a_1$ at time $t=1$, and $a_2' \neq a_2$ at $t=2$. This involves deriving the posterior $\mathbf{U}'=\mathbf{U} \mid \tau$ and predicting $s_2' \sim f(s_1,a_1',\mathbf{U}')$ and $s_3' \sim f(s_2',a_2',\mathbf{U}')$. At $t=1$, the observed transition informs the counterfactual transition probabilities (e.g., as we observed $s_2$, this state would be more likely in the counterfactual world). However, if the counterfactual state $s_2'$ diverges from the observed $s_2$ so that the distribution of the counterfactual outcome $s_3'$ is not at all influenced by the observation $\tau$, the counterfactual probability for the transition $s_2', a_2' \rightarrow s_3'$ will be equal to its prior (interventional) probability, as if no observation occurred. This way, the counterfactual analysis is no longer tailored to the individual observation. 

In this paper, we formalise this notion of \textit{counterfactual influence} in MDPs. In particular, we say that an observed path influences the counterfactual world if the interventional probabilities (i.e., the transition probabilities of the nominal MDP under the counterfactual policy) and the corresponding counterfactual probabilities are not identical. Using this concept, we can ensure our policy analysis is based on counterfactual paths that are sufficiently informed by the observation.

\paragraph{Motivating Example} To illustrate why influence is important for counterfactual analysis, consider Figure \ref{fig:sepsis-mdp-with-trajectories}. This depicts a subset of the state space for the Sepsis MDP \citep{oberst2019counterfactual}, which simulates trajectories of sepsis patients. \jl{Each dot represents a unique state (which are unordered),
and each row corresponds to a time step in an MDP trajectory.} When treating sepsis, it is important to consider whether the patient has diabetes, as diabetic patients may respond differently to treatment (e.g., with more varying blood glucose levels) than non-diabetic patients. Figure \ref{fig:sepsis-mdp-with-trajectories} illustrates this difference in terms of how often both groups visit each MDP state: the spectrum from red to blue represents how frequently states are reached by diabetic patients (in red), compared with the whole population (in blue), and the intensity of the colour represents how frequently states are visited by both groups.

\begin{figure}[ht]
    \centering
    \begin{minipage}{0.50\textwidth} 
        \centering
        \resizebox{1.0\columnwidth}{!}{%
            \begin{tikzpicture}
                \node[anchor=north west] (svg) {\includesvg{imgs/plotwoarrow.svg}};

                \draw[-{Latex[length=4mm, width=2mm]}, thick] (0,0) -- (9,0) node[above, pos=0.5, yshift=0.0cm] {\Large MDP state space}; %
                \draw[-{Latex[length=4mm, width=2mm]}, thick] (0, 0) -- (0, -9) node[left, pos=0.5] {\Large t}; %
            \end{tikzpicture}
        }
    \end{minipage}%
    \hfill
    \begin{minipage}{0.5\textwidth}
        \resizebox{\textwidth}{!}{%
            \begin{tikzpicture}[scale=1.0, every node/.style={scale=1.0}]
              \draw[blue, line width=3pt] (3.5,1.0) -- (4.0,1.0);
              \node[right] at (4.0,1.0) {Unconstrained Counterfactual};

              \draw[black, line width=3pt] (-0.5, 1.0) -- (0.0,1.0);
              \node[right] at (0.0,1.0) {Observation};

              \draw[red, line width=3pt] (2.0,0.25) -- (2.5,0.25);
              \node[right] at (2.5,0.25) {Influenced Counterfactual};
            \end{tikzpicture}
        }
        \vspace{0.5cm} 
        \resizebox{\textwidth}{!}{ 
            \begin{tikzpicture}[scale=1.0, every node/.style={scale=1.0}]
                \node at (-3.5,1.0) {\parbox{4.0cm}{\centering States only visited by patients from whole population}};

                \shade[left color=blue, right color=red] (-1.5,0.75) rectangle (1.5,1.25);

                \node at (3.5,1.0) {\parbox{4.0cm}{\centering States only visited by diabetic patients}};
            \end{tikzpicture}
        }
        \vspace{0.5cm}
        \resizebox{\textwidth}{!}{ 
            \begin{tikzpicture}[scale=1.0, every node/.style={scale=1.0}]
                \node at (-3.5,0.0) {\parbox{4.0cm}{\centering State never visited in simulated trajectories}};

                \shade[left color=white, right color=black] (-1.5,-0.25) rectangle (1.5,0.25);

                \node at (3.5,0.0) {\parbox{4.0cm}{\centering State visited in every simulated trajectory}};
            \end{tikzpicture}
        }
    \end{minipage}

    \caption{Subset of the state space of the Sepsis MDP. The spectrum from blue to red represents how frequently the state appears in simulated paths of diabetic patients (in red) vs. the whole population (in blue). The intensity of the colour represents how frequently the states are visited in the simulated paths. The black line is an observed trajectory for a diabetic patient; the blue line is the unconstrained counterfactual generated for that path, and, in red, the influence-constrained counterfactual path. \jl{The unconstrained counterfactual path diverges further from the the observation than the influenced counterfactual, as the observation and influenced counterfactual reach the same state at multiple timesteps ($t=1$ and $t=6$), and the states along both paths are shaded with a similar hue and intensity of red, indicating these paths have comparable (high) likelihoods of occurring in diabetic patients (vs. the general population), unlike the unconstrained counterfactual which visits a completely disjoint set of states.}}
    \label{fig:sepsis-mdp-with-trajectories}
\end{figure}

Given an observed trajectory of a diabetic patient (the black line in Figure \ref{fig:sepsis-mdp-with-trajectories}), we can use counterfactual inference to find the optimal counterfactual path (with the highest cumulative reward), to determine how their treatment could have been improved. The optimal counterfactual path (blue line) loses the information from the observation that the patient is diabetic, as it visits states (in blue) that are typical for the population on average but unlikely for diabetic patients. However, by constraining the optimal counterfactual path to retain influence from the observation, we obtain a counterfactual path (red line) more likely for a diabetic patient. Importantly, both counterfactual paths significantly improve on the observed path, as the patient has fewer vitals out of range at most time steps. Still, only the influenced path is tailored to the observation, albeit slightly sub-optimal compared to the unconstrained path. \jl{Also, if we derive counterfactual policies based on areas of the counterfactual MDP that are not informed/influenced by the observation, we risk identifying policies that would be optimal for the general population and not optimal for the observed (diabetic) patient. In Appendix \ref{app: simulation evidence}, we provide simulation evidence of this effect in the Sepsis MDP.}

Although this is a simplified example (as the diabetic status of a patient would be known, and their treatment can be adjusted accordingly), there may be situations where there is an underlying difference between sub-populations (e.g., patients with different genotypes) which is unknown. In these cases, it is beneficial to restrict counterfactual paths to ensure they remain tailored to the observations. 

\textbf{Contributions.} 
Leveraging our formal definition of influence, we propose an approach to derive counterfactual policies that are not only optimal for a given reward structure, but also maintain a degree of influence from the observed trajectory. This ensures that the counterfactuals remain tailored to the given observation. We build on recent work by~\citet{oberst2019counterfactual} which introduces Gumbel-max SCMs, a class of causal models for counterfactual inference of discrete-state MDPs, and~\citet{tsirtsis2021counterfactual}, which leverages Gumbel-max SCMs to derive an alternative sequence of actions that optimises the counterfactual outcome with a limit on the number of actions that can be changed. However, these and other existing papers on counterfactual analysis of MDPs, e.g. \citep{buesing2018woulda, lu2020sample,kazemi2023causal}, ignore the issue of influence, resulting in counterfactuals that may be (erroneously) disconnected from the observation. \jl{To the best of our knowledge, we are the first to identify this issue of counterfactual influence in sequential decision-making processes, and our paper is intended to support the expanding field of work applying counterfactual inference to MDPs by addressing this problem.} A thorough discussion of related work can be found in Appendix \ref{app:related work}. 

Procedurally, we build a counterfactual MDP that inherently satisfies influence constraints, through a polynomial-time algorithm that prunes the non-influenced transitions from the counterfactual MDP. We note that imposing this constraint may be too stringent, and results in counterfactual MDPs that are close or equal to the observed path. Therefore, we further extend the notion of influence to encompass multiple steps, so that influence constraints must hold at least once (as opposed to always) over paths of a given length. In this way, we can arbitrarily relax influence constraints, allowing for larger deviations from the observation, to favour optimality of the counterfactual policy.

We validate our approach on a Grid World model and two health-related case studies: an epidemic model and a sepsis model. We evaluate how relaxing the influence constraint impacts the derived optimal policy, and find that our approach results in 
counterfactual paths that are optimal or near-optimal while remaining influenced by the observed path. 
This results in more informative counterfactual explanations for improving a given policy, as these explanations are guaranteed to be tailored toward the given observation.

\section{Preliminaries\label{sec:prelim}}
MDPs are a class of stochastic models to describe sequential decision-making processes. In an MDP $\mdp$, at each step $t$, an agent in state $s_t$ performs some action $a_t$ determined by a policy $\pi$, ending up in state $s_{t+1} \sim \probs(s \mid s_t, a_t)$. The agent receives some reward $R(s_t,a_t)$ for performing $a_t$ at $s_t$. Formally, an MDP is a tuple $\mdp=(\states,\actions,\probs,P_I,R)$ where $\states$ is the discrete \emph{state space}, $\actions$ is the set of \emph{actions}, $\probs: (\states \times \actions \times \states) \rightarrow [0,1]$ is the \emph{transition probability} function, $P_I:\states \rightarrow [0,1]$ is the \textit{initial state distribution}, and $R: (\states \times \actions) \rightarrow \mathbb{R}$ is the \emph{reward function}. 
A (deterministic) \emph{policy} $\pi$ for $\mdp$ is a function $\pi:\states\rightarrow\actions$. A path $\tau$ of $\mdp$ under policy $\pi$, denoted $\tau \sim \mdp(\pi)$, is a sequence $\tau=(s_0, a_0), (s_1, a_1),\ldots,(s_{T-1},a_{T-1})$ where $T=|\tau|$ is the path length, $P_I(s_0)>0$, $a_t=\pi(s_t)$ for all $t=0,\ldots,T-1$, and $\probs(s_{t+1}\mid s_t, a_t)>0$ for all $t=0,\ldots,T-2$. 

While the MDP characterisation helps make predictions about future states and design action policies, it is not sufficient to make counterfactual predictions. For this, we require \textit{structural causal models (SCMs)}~\citep{pearl_2009}.  
Formally, an SCM is a tuple $\scm=(\exovars,\ndovars,\mathcal{F}= \{f_V\}_{V \in \ndovars},P(\exovars))$, where
$\exovars$ is a set of mutually independent \emph{exogenous variables} with $P(\exovars)$ being its distribution, and $\ndovars$ is a set of \emph{endogenous variables}. The value of each $V \in \ndovars$ is determined by a function $V = f_V(\mathbf{PA}_V, U_V)$ where $\mathbf{PA}_V \subseteq \ndovars$ is the set of direct causes of $V$ and $U_V\in \exovars$. %

Exogenous variables are unobserved and act as the source of randomness in the model. For a fixed realisation $\mathbf{u}\sim P(\exovars)$, i.e., a concrete unfolding of the system's randomness, the values of $\ndovars$ become deterministic, as they are uniquely determined by $\mathbf{u}$ and $\mathcal{F}$. 
The value $\mathbf{u}$ is also called \textit{context}. 
With SCMs, one can establish the causal effect of some input variable $X$ on some outcome $Y$ by evaluating $Y$ after applying an \textit{intervention} $X\gets x$, i.e., after replacing the RHS of $X = f_X(\mathbf{PA}_X, U_X)$ with $x$. We denote the resulting \textit{interventional distribution} by $P_{\scm[x]}(Y)$ (also written as $P(Y \mid \mathit{do}(x))$ in Pearl's \textit{do} notation).

Upon observing a realisation $\mathbf{v}$ of the SCM variables $\ndovars$, counterfactuals answer the following question: \textit{what would have been the value of variable $Y$ under intervention $X\gets x$, given that we observed  $\mathbf{v}$?} This corresponds to evaluating $\ndovars$ in a hypothetical world characterised by the same context that generated the observation $\mathbf{v}$ but under a different causal process (i.e., $\scm[x]$ instead of $\scm$).  
Computing counterfactuals first requires deriving the context that led to the observation, i.e., $\exovars \mid \mathbf{v}$, then evaluating the outcome under the counterfactual model obtained from $\scm[x]$ by replacing $P(\exovars)$ with the inferred $P(\exovars \mid \mathbf{v})$. 
Since each observation $\mathbf{v}$ can be seen as a deterministic function of a particular context $\mathbf{u}$, then in theory, the counterfactual outcome should be deterministic too. However, as we will see, $\mathbf{u}$ cannot be always uniquely identified from $\mathbf{v}$, resulting in a (non-Dirac) posterior distribution $P(\exovars \mid \mathbf{v})$ and a stochastic counterfactual outcome.  

\subsection{SCM-based Encoding of MDPs}\label{sec:scm_mdp_encoding}
\begin{wrapfigure}{r}{0.3\columnwidth}
    \centering
    \vspace{-4\intextsep}
            \resizebox{0.15\columnwidth}{!}{%
            \begin{circuitikz}
    \tikzstyle{every node}=[font=\Huge]

    \node[circle, draw, minimum size=2cm] (S_tp1) at (0, 0) {\Huge $S_{t+1}$};
    \node[circle, draw, minimum size=2cm] (S_t) at (-3, 1.5) {\Huge $S_t$};
    \node[circle, draw, minimum size=2cm] (A_t) at (-3, -1.5) {\Huge $A_t$};
    \node[circle, draw, dashed, minimum size=2cm] (U_t) at (0, 3) {\Huge $U_t$};

    \draw[->, thick, >=Stealth] (S_t) -- (S_tp1);
    \draw[->, thick, >=Stealth] (S_t) -- (A_t);
    \draw[->, thick, >=Stealth] (A_t) -- (S_tp1);
    \draw[->, thick, >=Stealth] (U_t) -- (S_tp1);
\end{circuitikz}
        }
        \caption{MDP causal graph}
    \label{fig:MDP_DAG}
\end{wrapfigure}
To enable counterfactual reasoning in MDPs, \citet{oberst2019counterfactual} proposed the following SCM-based encoding. For an MDP $\mdp$, policy $\pi$, and horizon $T$, we define an SCM over variables $\{S_t,A_t\}_{t=0}^{T-1}$ (which describes paths of length $T$ induced by $\mdp$ and $\pi$) with the following structural equations:
\begin{equation}\label{eq:mdp_scm}
    S_{t+1} = f(S_{t},A_{t}, U_{t}); \ \ A_t = \pi(S_t); 
    \ \ S_0 = f_0(U_0)
\end{equation}

This encoding allows us to compute counterfactuals, but it is not obvious how to define $f$ (and $f_0$) when dealing with categorical variables (arising from the MDP's discrete state space). To this purpose, \citet{oberst2019counterfactual} introduced \textit{Gumbel-Max SCMs}, expressed as
\begin{equation}\label{eq:gumbel-max-scm}
     S_{t+1} = f(S_{t},A_{t}, U_{t}=(G_{s,t})_{s\in \states}) = \argmax_{s\in \states}\left\{\log\left(\probs(s \mid S_t, A_t)\right) +G_{s,t}\right\}
\end{equation}
where for each $s\in \states$ and $t= 0,\ldots,T-1$, $G_{s,t}$ follows a Gumbel distribution. 
This approach is grounded in the Gumbel-Max trick, which shows that sampling from a categorical distribution with $k$ categories is equivalent to sampling $k$ copies $g_0,\ldots,g_k$ of a standard Gumbel and then determining the outcome as $\argmax_{j}\left\{\log\left(P(Y=j)\right)+ g_j\right\}$. 
Notably, the {Gumbel-Max SCM} encoding possesses a desirable feature called \textit{counterfactual stability}, which, informally, states that a counterfactual outcome remains equal to the observed outcome unless the intervention increases the relative probability of an alternative outcome\footnote{Counterfactual stability doesn't hold for instance if we encode the MDP using the inverse CDF trick.}. See \citep{oberst2019counterfactual} for more details. 

We note that, although the Gumbel-Max SCM is not the only SCM capable of expressing categorical distributions \citep{zhang2022partial}, such as those in MDPs, it remains the most popular causal model for MDPs \citep{lorberbom2021learning, benz2022counterfactual, noorbakhsh2022counterfactual, killian2022counterfactually, zhu2020counterfactual, NEURIPS2023_09ae6bea}.

\paragraph{Counterfactual inference.}
Given an MDP path $\tau=(s_0,a_0),\ldots,(s_{T-1},a_{T-1})$, counterfactual inference requires identifying the values of the Gumbel exogenous variables that align with $\tau$, i.e., calculate $\mathbf{G}'= (G_{s,t})_{s\in \states}^{t=0,\ldots,T-1}\mid \tau$. Because MDPs are Markovian, we can infer the Gumbel values for each observed transition in $\tau$ independently. However, because the mechanism of~\eqref{eq:gumbel-max-scm} is non-invertible (i.e., for given $s_t$ and $a_t$, multiple sets of $(G_{s,t})_{s\in \states}$ can result in the same $s_{t+1}$), we cannot uniquely identify the Gumbel values. Instead, as proposed by~\citet{oberst2019counterfactual}, we can achieve (approximate) posterior inference of $P((G_{s,t})_{s\in \states} \mid s_t,a_t,s_{t+1})$ through \textit{rejection sampling}: we draw samples from the prior $(g_{s,t})_{s\in \states}\sim P((G_{s,t})_{s\in \states})$ and discard all instances where $f(s_t,a_t,(g_{s,t})_{s\in \states})\neq s_{t+1}$. This can also be implemented more efficiently using the top-down Gumbel sampling approach described in~\citep{maddison2014sampling}.

\subsection{Counterfactual MDP and Optimal Policies}\label{sec:counterfactual mdp}
Given any MDP $\mdp$ (with known transition probabilities $P_\mdp$) and observed path $\tau$, we can define a corresponding \textit{counterfactual MDP} $\mdp^\tau$ which captures the counterfactual probabilities at any choice of state $s$ and action $a$. $\mdp^\tau$ will be a non-stationary MDP with the same state space $\states$, action space $\actions$, and reward structure $R$ as $\mdp$. Its initial state distribution $P^\tau_I$ assigns probability $1$ to $s_0$, the first state of $\tau$; and its transition probabilities directly follow from the SCM~\eqref{eq:gumbel-max-scm} and are defined, for $t=0,\ldots,T-1$ and $\forall s' \in \mathcal{S}$, as 

\begin{multline}
\label{eq:cf_mdp_probs}
     P_{\mdp,t,\tau}(s' \mid s, a) = P\left(s' = \argmax_{q \in \states} \left\{\log\left(\probs(q \mid s, a)\right) + 
     G'_{q,t}\right\}\right) \\
     \approx \dfrac{1}{N} \sum_{j=0}^{N} \mathbbm{1}\left(s' = \argmax_{q \in \states} \left\{\log\left(\probs(q \mid s, a)\right) + 
     G_{q,t}^{\prime(j)}\right\}\right)
\end{multline}
where we sample $N$ values $G_{q,t}^{\prime(j)}$ from the true posterior distribution $G'_{q,t}$ using either the rejection sampling or top-down sampling approach.  The indicator function $\mathbbm{1}(\mathbbm{X})$ takes the value $1$ if the condition $\mathbbm{X}$ is satisfied and $0$ otherwise. $\mdp^\tau$ can be directly solved to derive optimal counterfactual policies, as done in~\citep{tsirtsis2021counterfactual}. In that paper, the authors are concerned with finding optimal action sequences that deviate from the observed path by at most $m$ actions. We call this an \textit{$m$-CF policy}:

\begin{definition}[$m$-CF policy~\citep{tsirtsis2021counterfactual}]\label{def:m-policy}
    Let $\tau$ be a path of an MDP $\mdp$ of length $T$, and let $\mdp^\tau$ be the corresponding counterfactual MDP. For a given $m\leq T$, an \emph{$m$-CF policy} $\pi^*$ is one that maximises the value $V_{\tau}(\pi)=\mathbb{E}_{\tau' \sim \mdp^\tau(\pi)} \left[ \sum_{t=0}^{T-1} R(s'_t, a'_t) \right]$ under the condition that any counterfactual path $\tau'$ induced by $\mdp^\tau$ and $\pi^*$ satisfies $\sum_{t=0}^{T-1} \mathbbm{1}(a_t \neq a'_t) \leq m$.
\end{definition}

To derive $m$-CF policies, \citet{tsirtsis2021counterfactual} employ a polynomial-time dynamic programming algorithm that keeps track of the number of action changes between the observed path $\tau$ and the counterfactual one $\tau'$. However, their approach has two main shortcomings. First,  $\tau'$ may diverge from $\tau$, visiting different state sequences where it may not be sensible to impose the same observed action. Second, and most important, because of this divergence, it is possible to reach a state in $\tau'$ where the observed path $\tau$ bears no longer influence, meaning that we are no longer computing counterfactuals but interventional outcomes. In the next sections, we characterise and solve the latter issue.

\section{Theoretical Framework}
\label{sec: theory}
This section introduces a formal notion of \emph{influence}, to describe how an observed path $\tau$ affects the counterfactual world. For counterfactual MDPs derived using Gumbel-max SCMs (see \eqref{eq:cf_mdp_probs}), the counterfactual 
distribution $P_{\mdp,t,\tau}$ will be identical to the nominal/interventional distribution $P_{\mdp}$ when there is no influence. Before formally defining influence, we first derive a sufficient condition for these two distributions to be equal when using Gumbel-max SCMs:

\begin{proposition}
Let $\tau$ be a path of an MDP $\mdp$ of length $T$, and let $\mdp^\tau$ be the corresponding counterfactual MDP. Given a time $t<T$  and counterfactual state $s'_t$ and action $a'_t$ in $\mdp^\tau$, then if $P_{\mdp}(\cdot \mid s_t, a_t)$ and $P_{\mdp}(\cdot \mid s'_t, a'_t)$ have disjoint support, then the counterfactual distribution $P_{\mdp,t,\tau}(\cdot \mid s'_t, a'_t)$ is identical to the interventional distribution $P_{\mdp}(\cdot \mid s'_t, a'_t)$. 
\label{prop: no influence disjoint supports}
\end{proposition}

\begin{proof}
At time $t$, if the distributions $P_{\mdp}(\cdot \mid s_t, a_t)$  and $P_{\mdp}(\cdot \mid s'_t, a'_t)$ have disjoint supports then the posterior Gumbel distribution relative to the possible next states from $(s'_t, a'_t)$ remains the same as the prior Gumbel distribution for those states. This is because the states in the support of $P(\cdot \mid s'_t, a'_t)$ cannot be reached from $(s_t, a_t)$ in the factual world (i.e., in the original MDP), hence, no observation can change their prior. As a consequence,  by~\eqref{eq:gumbel-max-scm} and~\eqref{eq:cf_mdp_probs}, the counterfactual probabilities for those states remain equal to the probabilities in the original MDP. 
\end{proof}

With this proposition, we can define influence by checking a simple condition on the transition probabilities of the original MDP. Importantly, this condition is precise because the probabilities $P_{\mdp}$ are known. On the other hand, the equality $P_{\mdp,t,\tau}(\cdot \mid s'_t, a'_t) = P_{\mdp}(\cdot \mid s'_t, a'_t)$ cannot be established precisely because of the sampling error introduced in $P_{\mdp,t,\tau}$ by the Gumbel posterior inference\footnotemark.

\footnotetext{It is possible to define other notions of influence that are more quantitative, e.g., based on the statistical distance between the transition probabilities. This is discussed in more detail in Appendix \ref{app: probabilistic influence}.}

\begin{definition}[1-step influence]\label{def:immediate_influence}
Let $\tau$ be a path of an MDP $\mdp$ of length $T$, and let $\mdp^\tau$ be the corresponding counterfactual MDP. Given a time $t<T$ and counterfactual state $s'_t$ and action $a'_t$ in $\mdp^\tau$, we say that $\tau$ exerts an immediate ($1$- step) influence on $s'_t$ and $a'_t$ at time $t$ if and only if the supports of the distributions $P_{\mdp}(\cdot \mid s'_t, a'_t)$ and $P_{\mdp}(\cdot \mid s_t, a_t)$ are not disjoint. 
\end{definition}

While imposing influence constraints is clearly desirable, the above notion of $1$-step influence may be too strict and we may not be able to deviate much, or at all, from the observed path. To overcome this potential limitation, we relax and generalise the notion of $1$-step influence to encompass multiple steps, such that influence constraints must hold at least once in a counterfactual path, and not at every step\footnote{\jl{The issue of (lack of) influence is especially pronounced in sequential processes (that over time, may drift away from the observed path). However, this problem can also in a static (i.e., $1$-step) setting. Consider an action $a’$ that, when applied to the initial observed state $s_0$ (rather than the observed action $a_0$) leads to a next state distribution $P(\cdot | s_0, a’)$ that has disjoint support w.r.t. $P(\cdot \mid s_0, a_0)$. By Proposition \ref{prop: no influence disjoint supports}, this implies that the counterfactual distribution will be equal to the interventional one, i.e., the observation $(s_0,a_0,s_1)$ doesn’t inform the counterfactual outcome for $(s_0,a’)$.}}.

\begin{definition}[$k$-step influence]\label{def:k_influence}
Let $\mdp$, $T$, $\tau$, and $\mdp^\tau$ be as in Definition~\ref{def:immediate_influence}. Given a time $t<T$, horizon $k$, and counterfactual state $s'_t$ and action $a'_t$, if $t+k \leq T$ we say that $\tau$ exerts a $k$-step influence on $s'_t$ and $a'_t$ at time $t$ if there exists a path $\tau'$ of $\mdp^\tau$ of length $k$ starting in $(s'_t,a'_t)$ and such that $\tau$ exerts a $1$-step influence on at least one state of $\tau'$. If $t+k > T$, $\tau$ always exerts a $k$-step influence on $s'_t$ and $a'_t$ at time $t$.
\end{definition}

\begin{remark}
\label{remark:k-influence}
    In the above definition, when $t+k > T$, the length $T$ of the observed path is not sufficient to determine $k$-step influence, and so we make the conservative assumption that influence holds in such cases. In particular, for $k\geq T+1$, we have that the observed path $\tau$ exerts a $k$-step influence on any counterfactual state at any time $t$ (because when $k\geq T+1$, then $t+k > T$ for any $t$).
\end{remark}
Figures~\ref{fig:one-step} and ~\ref{fig:two-step} \jl{depict a rollout of} a toy counterfactual MDP with two possible actions, $a_0$ and $a_1$, at states $s_0$ and $s_3$, and only action $a_0$ at the rest of the states. Given the observed path $\tau = [(s_0, a_0), (s_2, a_0), (s_5, a_0), (s_7, a_0)]$, we want to find the influence-constrained counterfactual MDP given $k=1$ and $k=2$. When $k=1$ (Figure~\ref{fig:one-step}), $(s_1, a_0)$ and $(s_3, a_1)$ are not influenced at $t=1$, because they have disjoint support with the observed state-action pair $(s_2, a_0)$. 
For the opposite reason, we have that $(s_2, a_0)$ and $(s_3, a_0)$ are influenced at $t=1$. $(s_4, a_0)$, $(s_5, a_0)$ and $(s_6, a_0)$ are all influenced at $t=2$, as these have overlapping support with the observed pair $(s_5, a_0)$. However, even though $(s_4, a_0)$ and $(s_6, a_0)$ are influenced, they cannot be reached from any influenced state-action pairs, so are also removed from the influence-constrained counterfactual MDP: we say they are ``influenced but unreachable''.
Figure \ref{fig:two-step} depicts the case of 2-step influence. We note that $(s_6, a_0)$ is now reachable, because $(s_3, a_1)$ is influenced at $t=1$ with 2-step influence. However, even though $(s_1, a_0)$ now becomes influenced at $t=1$, it cannot be reached by any influenced state-action pair, so $(s_1, a_0)$ and $(s_4, a_0)$ are influenced but unreachable.

\begin{figure}[ht]
    \centering
    \resizebox{0.45\textwidth}{!}{ %
    \begin{tikzpicture}
        \matrix [draw, below=1cm of current bounding box.south] {
            \node [main node, fill=black!50,label=right:Observed] (l1) {}; \pgfmatrixnextcell
            \node [main node, fill=black!25,label=right:Influenced and reachable] (l2) {}; \\
            \node [main node,label=right:Influenced but unreachable] (l3) {}; \pgfmatrixnextcell
            \node [main node, gray,label=right:Not $k$-step influenced] (l4) {}; \\
        };
    \end{tikzpicture}
    }

    \subfigure[$k=1$]{%
         \resizebox{0.25\textwidth}{!}{\begin{tikzpicture}[->,>=stealth,node distance=1.2cm,thick,main node/.style={circle,draw}, ]

      \node[main node, fill=black!50] (s0) {$s_0$};
      \node[main node, fill=black!50] (s2) [right of=s0] {$s_2$};
      \node[main node, gray] (s1) [below of=s2] {$s_1$};
      \node[main node, fill=black!25] (s3) [above of=s2] {$s_3$};
      \node [above] at (s3.north) (t1) {t=1};
      \node [] (t0) [left of=t1] {t=0}; %
      \node[main node] (s4) [right of=s1] {$s_4$};
      \node[main node, fill=black!50] (s5) [right of=s2] {$s_5$};
      \node[main node] (s6) [right of=s3] {$s_6$};
      \node [above] at (s6.north) (t2) {t=2};
      \node[main node, fill=black!50] (s7) [right of=s5] {$s_7$};
      \node [] (t3) [right of=t2] {t=3}; %
      \draw[gray] (s0) -- (s1) node[midway,right] {$a_1$};
      \draw[ultra thick] (s0) -- (s2) node[midway, above] {$a_0$};
      \draw (s0) -- (s3) node[midway,right] {$a_0$};
      \draw[gray] (s1) -- (s4) node[midway,above] {$a_0$};
      \draw[ultra thick] (s2) -- (s5) node[midway,above] {$a_0$};
      \draw[gray] (s3) -- (s6) node[midway, above] {$a_1$};
      \draw (s3) -- (s5) node[midway, right] {$a_0$};
      \draw (s4) -- (s7) node[midway, right] {$a_0$};
      \draw[ultra thick] (s5) -- (s7) node[midway, above] {$a_0$};
      \draw (s6) -- (s7) node[midway, right] {$a_0$};

    \end{tikzpicture}}
        \label{fig:one-step}
    }
    \hspace{0.05\textwidth} %
    \subfigure[$k=2$]{%
         \resizebox{0.25\textwidth}{!}{\begin{tikzpicture}[->,>=stealth,node distance=1.2cm,thick,main node/.style={circle,draw}]

      \node[main node, fill=black!50] (s0) {$s_0$};
      \node[main node, fill=black!50] (s2) [right of=s0] {$s_2$};
      \node[main node] (s1) [below of=s2] {$s_1$};
      \node[main node, fill=black!25] (s3) [above of=s2] {$s_3$};
      \node [above] at (s3.north) (t1) {t=1};
      \node [] (t0) [left of=t1] {t=0}; %

      \node[main node] (s4) [right of=s1] {$s_4$};
      \node[main node, fill=black!50] (s5) [right of=s2] {$s_5$};
      \node[main node, fill=black!25] (s6) [right of=s3] {$s_6$};
      \node [above] at (s6.north) {t=2};
      \node[main node, fill=black!50] (s7) [right of=s5] {$s_7$};
      \node [] (t3) [right of=t2] {t=3}; %
      \draw[gray] (s0) -- (s1) node[midway,right] {$a_1$};
      \draw[ultra thick] (s0) -- (s2) node[midway, above] {$a_0$};
      \draw (s0) -- (s3) node[midway,right] {$a_0$};
      \draw (s1) -- (s4) node[midway,above] {$a_0$};
      \draw[ultra thick] (s2) -- (s5) node[midway,above] {$a_0$};
      \draw (s3) -- (s6) node[midway, above] {$a_1$};
      \draw (s3) -- (s5) node[midway, right] {$a_0$};
      \draw (s4) -- (s7) node[midway, right] {$a_0$};
      \draw[ultra thick] (s5) -- (s7) node[midway, above] {$a_0$};
      \draw (s6) -- (s7) node[midway, right] {$a_0$};

\end{tikzpicture}}
        \label{fig:two-step}
    }

    \caption{Example counterfactual MDP given $k$-step influence. State-action pairs may or may not be influenced by the observed path, and states may or may not be reachable from other influenced state-action pairs.}
\end{figure}

We provide more details on the construction of the counterfactual MDPs in Figures \ref{fig:one-step} and \ref{fig:two-step} in Appendix \ref{app: influence constrained mdps}. We now re-formulate the idea of optimal counterfactual policy by incorporating our notion of influence. Previously, in Definition~\ref{def:m-policy}, we restricted to policies ensuring a bounded number of action changes from the observed action sequence. 
Here, we further include constraints to guarantee that any counterfactual path is influenced to some degree by the observation.  

\begin{definition}[$(k, m)$-CF policy]
Let $\tau$ be a path of an MDP $\mdp$ of length $T$, and let $\mdp^\tau$ be the corresponding counterfactual MDP. For a given $m\leq T-1$ and influence bound $k$, a \emph{$(k, m)$-CF policy} $\pi^*$ is one that maximises the value $V_{\tau}(\pi)=\mathbb{E}_{\tau' \sim \mdp^\tau(\pi)} \left[ \sum_{t=0}^{T-1} R(s'_t, a'_t) \right]$ under two conditions:  1) the observed path $\tau$ exerts a $k$-step influence on any counterfactual path $\tau'$ induced by $\mdp^\tau$ and $\pi^*$; and 2) any such counterfactual path satisfies the constraint $\sum_{t=0}^{T-1} \mathbbm{1}(a_t \neq a'_t) \leq m$.
\end{definition}
By Remark \ref{remark:k-influence}, if $k \geq T+1$, then a $(k,m)$-CF policy corresponds exactly to a $m$-CF policy, as defined in Definition \ref{def:m-policy}, because the influence constraint will always be trivially satisfied ($t + k > T$, for all $t$).

\section{Methodology}
\label{sec:methodology}
In this section, we describe the steps of our algorithm for finding the optimal $(k, m)$-CF policy for a given MDP. The pseudocode for the whole algorithm can be found in Appendix \ref{app:pseudocode}. To achieve this, we first constrain the counterfactual MDP to only transitions that are influenced by the observed path (under a given $k$-step influence), then apply value iteration to find the optimal policy in the influence-constrained counterfactual MDP. Both these steps have polynomial complexity, and so the whole algorithm remains polynomial as well.

First, we calculate the counterfactual transition probabilities $P_{\mdp,t,\tau}$ for all transitions using the top-down Gumbel sampling approach \citep{maddison2014sampling}, as described in Sections \ref{sec:scm_mdp_encoding} and \ref{sec:counterfactual mdp}. To construct the influence-constrained counterfactual MDP, we first identify, for each state-action $(s_t, a_t)$ pair in the observed path $\tau$, all states $s$ that are in the support of $P(\cdot \mid s_t, a_t)$. We denote the set of such states with $S^\tau_t$, and denote the union of all of these $S^\tau_t$ as $S^\tau = \bigcup_{t=0}^{T-1} S^\tau_t$. For the example MDP given in Figures~\ref{fig:one-step} and~\ref{fig:two-step}, $S^\tau = \{s_2, s_3, s_5, s_7\}$. Next, we execute a reverse Breadth-First Search (BFS) algorithm with a maximum depth of $k$ over the original MDP, starting from each state in $S^\tau$. This identifies the set $S^{\tau, k}$ of MDP states that can reach, within $k$ steps, a state in $S^\tau$, i.e., a state that is influenced by $\tau$. As shown in Figure \ref{fig:one-step}, $S^{\tau, 1} = \{s_0, s_2, s_3, s_4, s_5, s_6, s_7\}$, and, as shown in Figure \ref{fig:two-step}, $S^{\tau, 2} = \{s_0, s_1, s_2, s_3, s_4, s_5, s_6, s_7\}$. Since the BFS algorithm is polynomial in the number of states and transitions, and we run the algorithm $|S^\tau|\leq |S|$ times, the worst-case computational complexity of the algorithm remains polynomial ($O(|S^{\tau}| \cdot (|S| + (|S|^2 \cdot |A|)))$).

To prune the non-influenced state-action pairs, we set $P_{\mdp,t,\tau}(\cdot \mid s, a)=0$ whenever there exists a $s'$ where $P(s' \mid s, a) > 0$ in the original MDP, $s' \not \in S^{\tau, k}$ and $t + k \leq T$ (as stated in Definition \ref{def:k_influence}, all state-action pairs are influenced by the observation when $t+k > T$). This ensures that no action can lead to states outside the influence-constrained counterfactual MDP. Finally, we must prune any states that are unreachable or have no outgoing transitions, by setting $P_{\mdp,t,\tau}(\cdot \mid s, a)=0$ for all pairs $(s,a)$ that have some probability of leading to these states. As shown in Figures \ref{fig:one-step} and \ref{fig:two-step}, this last pruning step removes states $s_4$ and $s_6$ for $k=1$, and states $s_1$ and $s_4$ for $k=2$, as these states are not reachable.

Given the influence-constrained counterfactual MDP, we apply a value iteration algorithm (similar to the method in \citep{tsirtsis2021counterfactual}) to find the optimal $(k, m)$-CF policy, with the only change being that the action choices are restricted to only those in the influence-constrained counterfactual MDP. The worst-case complexity of this dynamic programming algorithm for given values of $k$ and $m$ is $O(|\tau| \cdot |S|^2 \cdot |A| \cdot m)$, when all state-action pairs from the MDP $\mathcal{M}$ are contained in the influence-constrained counterfactual MDP. However, in practice, the size of the influence-constrained counterfactual MDP will often be significantly smaller than the unconstrained counterfactual MDP, reducing the complexity and execution time of the value iteration algorithm. The sizes of the counterfactual MDPs in our experiments, pruned with each value of $k$, are provided in Table \ref{tab:state_space_after_pruning}.

\jl{Our approach is always guaranteed to identify the optimal $(k,m)$-CF policy, as follows:
\begin{theorem}[Optimal $(k,m)$-CF Policy Guarantee]
    For any given MDP $\mathcal{M}$, observed path $\tau$, and values of $k$ and $m$, our method is guaranteed to identify the optimal $(k,m)$-CF policy.
\end{theorem}
\begin{proof}
\mk{By construction, }the pruning step of our algorithm ensures that only the parts of the counterfactual MDP that satisfy $k$-step influence remain. We then apply the value iteration algorithm from \citep{tsirtsis2021counterfactual} to the influence-constrained counterfactual MDP to determine the optimal $m$-CF policy. This is a standard value iteration algorithm based on dynamic programming, which is guaranteed to find the optimal policy \citep{sutton2018reinforcement}. Therefore, because the optimal $m$-CF policy is derived over only the transitions that satisfy the $k$-step influence constraint, our method is guaranteed to identify the optimal $(k,m)$-CF policy.
\end{proof}

}
\section{Experiments}
\label{sec:experiments}
In this section, we apply our notion of influence to several MDPs to derive counterfactual policies with varying levels of influence with respect to the observed path. We evaluate our approach on a Grid World model, an epidemic MDP model, and an MDP modelling sepsis patient trajectories\footnote{The source code for our experiments can be found \jl{at \url{https://github.com/ddv-lab/counterfactual-influence-in-MDPs}}.}.

\subsection{Setup}
For each MDP, we first generate an observed path of length $T$ using a deliberately suboptimal policy. The choice of $T$ depends on the MDP being evaluated to ensure there is sufficient opportunity to improve the policy, and is noted explicitly in the following subsections. We use this observed path to generate the optimal $(k, m)$-CF policy for $1 \leq k \leq T+1$ and $1 \leq m \leq T$, using the algorithm described in Section \ref{sec:methodology}. We assume that counterfactual paths begin in the same initial state as the observed path. To evaluate the performance of the $(k, m)$-CF policies, we consider the value of the initial state, $V(s_0)$, as this measures how good the paths generated by the optimal policy (starting at the initial state) will be. As a baseline, we consider the $m$-CF policies produced by the algorithm in \citep{tsirtsis2021counterfactual}, as this is the only other existing method for counterfactual inference in MDPs - this baseline is reported in our results as $k=T+1$.
\subsection{Grid World}
\begin{figure}
    \centering
    \begin{tikzpicture}[scale=0.7, transform shape]
{\Large
\begin{axis}[
    xlabel={Maximum Number of Actions Changed},
    ylabel={V($s_0$)},
    xmin=0, xmax=12.5,
    ymin=-1000, ymax=600,
    xtick={0,2,4,6,8,10,12},
    ytick={-1000, -800,-600,-400,-200,0, 200, 400, 600},
    legend style={at={(0.9,0.5)}, anchor=east}, %
    ymajorgrids=true,
    grid style=dashed,
    width=10cm,
]

\addplot[
    color=pink,
    mark=*,
    only marks,
    mark options={fill=pink, draw=pink, opacity=1.0},
    mark size=3pt
]
coordinates {
    (1,-897) (2,-897) (3,-897) (4,-897) (5,-897) (6,-897) (7,-897) (8,-897) (9,-897) (10,-897) (11,-897) (12,-897)
};

\addplot[
    color=blue,
    mark=x,
    only marks,
    mark size = 3pt
]
coordinates {
    (1,-897) (2,-897) (3,-897) (4,-897) (5,-897) (6,-897) (7,-897) (8,-897) (9,-897) (10,-897) (11,-897) (12,-897)
};

\addplot[
    color=green,
    mark=diamond*,
    only marks,
    mark size=3pt
]
coordinates {
    (1,-12.493350733858055) (2,4.194131911297905) (3,37.27091871435141) (4,361.23427100610456) (5,421.180039320056) (6,477.08103652315737) (7,519.5622206868214) (8,534.4494547495025) (9,540.5475952932438) (10,541.6044338969365) (11,541.6050142756363) (12,541.6050142756363)
};

\legend{Observed Path,K=1 to K=6,K=7 to K=T+1}

\end{axis}}
\end{tikzpicture}
    \caption{Grid World: value of initial state given $k$-step influence and maximum $m$ actions changed}
    \label{fig:gridworld-value}
\end{figure}
In the Grid World experiment, the agent must traverse a 4x4 grid from the top-left corner to the bottom-right corner, avoiding a dangerous terminal state in the middle of the grid. At each time step, the agent can choose to move up, down, left, or right. However, there is a small probability that the agent will move in a different direction to the chosen action. As the agent gets closer to the goal state in the bottom-right corner, the reward for each state increases, and the agent receives a reward of $100$ for reaching the terminal goal state. However, there is also a reward of $-100$ for transitioning into the dangerous terminal state.

We derive the counterfactual MDP using an observed path of length $11$ that falls into this dangerous terminal state at $t=3$. Figure \ref{fig:gridworld-value} shows how $k$ and $m$ affect the value collected by the $(k,m)$-CF policy. When $k \leq 6$, we see no improvement in the policy, because all influenced counterfactual paths lead to the dangerous terminal state. However, when $k \geq 7$, the optimal influenced counterfactual path avoids the dangerous state for all $m$ (shown by $V(s_0) \geq 0$), and when $m \geq 4$ it reaches the goal state and gains the $+100$ reward, resulting in much higher values for the initial state. This shows that we don't need to sacrifice optimality to generate counterfactual paths which are still influenced by the observed path, because the optimal counterfactual paths for $k=7$ and $k=T+1$ both reach the goal state, but the counterfactual path for $k=7$ is more informed by the observation than the path for $k = T+1$ (i.e., obtained without influence constraints).

\subsection{Epidemic Model}
The epidemic MDP simulates how infection spreads through a (discrete) population. Each state $(S, I, V)$ consists of $S$ susceptible and $I$ infected individuals, and the number of available vaccines $V$. At each time step, the agent implements a vaccination strategy and can choose among three actions: do nothing, vaccinate an infected individual, or vaccinate a susceptible individual. The reward for each transition $(s, a, s')$ is given by the negative value of the number of infected individuals in $s$, $-I$. Full details of the model are given in Appendix \ref{app:epidemic_mdp}.

The observed path of length $7$ begins in state $(S=9, I=1, V=20)$ and is generated from a suboptimal policy that chooses to ``do nothing'' in every state. Therefore, we can generate increasingly better counterfactual paths that are still highly influenced by the observed path, because switching the action in most states would lead to an improved outcome. This is shown in Figure \ref{fig:epidemic-value}, where for a $k$-step influence with $k > 2$, the policy value increases monotonically with $k$ and $m$. Figure \ref{fig:epidemic-infection-rates} shows the average number of infected individuals $I$ at each time step $t$ for selected combinations of $(k, m)$, compared to the observed path, for $1000$ simulated trajectories. This further illustrates how relaxing the influence constraint impacts the final reward. 
With $k=3$, we obtain a counterfactual path almost identical to the observed one except for the last two steps. On the other hand, when no influence constraints are imposed ($k=T+1$), we obtain the optimal counterfactual path where the first action is changed to vaccinate the single infected person and stop the epidemic.
Note that when the counterfactual paths follow the observed path (i.e., the counterfactual state coincides with the observed state, and the observed action is taken), the transitions become deterministic (leading to the next observed state with probability$=1$), as we can see in the plot for some path prefixes. 

\begin{figure}[tbh]
    \centering
    \subfigure[Epidemic MDP: initial state value given changing influence]{%
    \resizebox{0.45\linewidth}{!}{
        \begin{tikzpicture}[scale=1.0, baseline=(current bounding box.south)]
{\Large
\begin{axis}[
    xlabel={Maximum Number of Actions Changed},
    ylabel={V($s_0$)},
    xmin=0.5, xmax=7.5,
    ymin=-40, ymax=0,
    xtick={1,2,3,4,5,6,7},
    ytick={-40,-35,-30,-25,-20,-15,-10, -5, 0},
    legend style={
        at={(0.5,1.05)},  %
        anchor=south,      %
        legend columns=4, %
        /tikz/every even column/.append style={column sep=0.2cm}
    },
    ymajorgrids=true,
    grid style=dashed,
    cycle list name=color list,
    every axis plot/.append style={thick},
    width=10cm,
]

\addplot+[
    color=pink,
    mark=*,
    only marks,
    mark options={fill=pink, opacity=1.0, mark size=3pt}
]
coordinates {
    (1,-38.0) (2,-38.0) (3,-38.0) (4,-38.0) (5,-38.0) (6,-38.0) (7,-38.0)
};
\addlegendentry{Observed Path}

\addplot+[
    color=blue,
    mark=x,
    only marks,
    mark size=3pt
    ]
coordinates {
    (1,-38.0) (2,-38.0) (3,-38.0) (4,-38.0) (5,-38.0) (6,-38.0) (7,-38.0)
};
\addlegendentry{K=1 to K=2}

\addplot+[
    color=orange,
    mark=square*,
    only marks,
    ]
coordinates {
    (1,-37.109) (2,-37.109) (3,-37.109) (4,-37.109) (5,-37.109) (6,-37.109) (7,-37.109)
};
\addlegendentry{K=2}

\addplot+[
    color=yellow,
    mark=triangle*,
    only marks,
    ]
coordinates {
    (1,-35.154887) (2,-34.203283) (3,-34.203283) (4,-34.203283) (5,-34.203283) (6,-34.203283) (7,-34.203283)
};
\addlegendentry{K=3}

\addplot+[
    color=green,
    mark=diamond*,
    only marks,
    ]
coordinates {
    (1,-33.906211194) (2,-32.090256487999994) (3,-31.130102333) (4,-31.130102333) (5,-31.130102333) (6,-31.130102333) (7,-31.130102333)
};
\addlegendentry{K=4}

\addplot+[
    color=red,
    mark=pentagon*,
    only marks,
    ]
coordinates {
    (1,-29.646347838121) (2,-26.16431996308) (3,-23.990685078201) (4,-23.026639107063) (5,-23.026639107063) (6,-23.026639107063) (7,-23.026639107063)
};
\addlegendentry{K=5}

\addplot+[
    color=purple,
    mark=*,
    only marks,
    ]
coordinates {
    (1,-27.907975807272233) (2,-19.953992664651857) (3,-14.406184200251841) (4,-11.4607933279152) (5,-10.386358523854401) (6,-10.386358523854401) (7,-10.386358523854401)
};
\addlegendentry{K=7}

\addplot+[
    color=black,
    mark=otimes,
    only marks,
    ]
coordinates {
    (1,-1.0) (2,-1.0) (3,-1.0) (4,-1.0) (5,-1.0) (6,-1.0) (7,-1.0)
};
\addlegendentry{K=T+1}

\end{axis}}
\end{tikzpicture}
        \label{fig:epidemic-value}
        }
    }
    \hspace{0.5cm} %
    \subfigure[Epidemic MDP: mean infection rate over time for select $(k, m)$. The error bars show $\sigma(I)$.]{%
    \raisebox{0.12cm}{\resizebox{0.45\linewidth}{!}{
        \begin{tikzpicture}[scale=1.0, baseline=(current bounding box.south)]
{\Large
    \begin{axis}[
        xlabel={$t$},
        ylabel={Mean $I$},
        xmin=0, xmax=7,
        ymin=0, ymax=10,
        xtick={0,1,2,3,4,5,6,7},
        ytick={0,2,4,6,8,10},
        width=10cm,
        grid=major,
        legend style={
        at={(0.5,1.15)},  %
        anchor=south,      %
        legend columns=5, %
        /tikz/every even column/.append style={column sep=0.05cm},
    }
    ]
    
    \addplot[mark=*, black, error bars/y dir=both, error bars/y explicit] coordinates {
        (0,1) +- (0,0)
        (1,2) +- (0,0)
        (2,3) +- (0,0)
        (3,6) +- (0,0)
        (4,8) +- (0,0)
        (5,9) +- (0,0)
        (6,9) +- (0,0)
        (7,9) +- (0,0)
    };
    \addlegendentry{Observed Path}

    \addplot[mark=*, blue, error bars/y dir=both, error bars/y explicit] coordinates {
        (0,1) +- (0,0)
        (1,2) +- (0,0)
        (2,3) +- (0,0)
        (3,6) +- (0,0)
        (4,8) +- (0,0)
        (5,9) +- (0,0)
        (6, 8.118) +- (0,0.32260812)
        (7, 8.208)+- (0,0.40587683)
    };
    \addlegendentry{$(3, 1)$}

    \addplot[mark=*, red, error bars/y dir=both, error bars/y explicit] coordinates {
        (0,1) +- (0,0)
        (1,2) +- (0,0)
        (2,3) +- (0,0)
        (3,3.548) +- (0,0.55289782)
        (4,4.191) +- (0,0.9739194)
        (5,4.618) +- (0,1.113587)
        (6,4.482) +- (0, 0.98978584)
        (7,4.922)+- (0,0.84847864)};
    \addlegendentry{$(6, 4)$}

    \addplot[mark=*, green, error bars/y dir=both, error bars/y explicit] coordinates {
        (0,1) +- (0,0)
        (1,2) +- (0,0)
        (2,1.91) +- (0,0.28618176)
        (3,1.704) +- (0,0.62320462)
        (4,1.486) +- (0,0.81351337)
        (5,1.269) +- (0,0.90146492)
        (6,2.163) +- (0,1.58506498)
        (7,2.955)+- (0,2.09689652)};
    \addlegendentry{$(7, 4)$}

    \addplot[mark=*, orange, error bars/y dir=both, error bars/y explicit] coordinates {
        (0,1) +- (0,0)
        (1,0) +- (0,0)
        (2,0) +- (0,0)
        (3,0) +- (0,0)
        (4,0) +- (0,0)
        (5,0) +- (0,0)
        (6,0) +- (0,0)
        (7,0) +- (0,0)
    };
    \addlegendentry{$(8, 1)$}
    
    \end{axis}}
\end{tikzpicture}
        \label{fig:epidemic-infection-rates}
        }
    }}
    \caption{Epidemic MDP analysis results}
\end{figure}

\subsection{Sepsis Model\label{subsec:sepsis}}
The Sepsis MDP is taken from \citep{oberst2019counterfactual}\footnote{Licensed under the MIT License, and available at \url{https://github.com/clinicalml/gumbel-max-scm}.} and models the trajectories of sepsis patients. Each state consists of four vital signs (heart rate, blood pressure, oxygen concentration, and glucose levels), with possible values of low, normal, or high. There are three treatment options, which can be turned on or off at each time step (8 actions in total). Unlike in \citep{oberst2019counterfactual}, we scale the rewards depending on the number of vital signs that are out of range, between $-1000$ (patient dies) and $1000$ (patient is discharged). For further details, we refer to \citep{oberst2019counterfactual}. In our experiments, we simulated the trajectory of a sepsis patient over $10$ time steps.

\begin{figure}[tbh]
    \centering
    \subfigure[Sepsis: $V(s_0)$ given $k$-step influence and maximum $m$ actions changed for a catastrophic path.]{%
       \resizebox{0.45\linewidth}{!}{%
            \begin{tikzpicture}[scale=1.2, baseline=(current bounding box.south)]
\begin{axis}[
    xlabel={Maximum Number of Actions Changed},
    ylabel={V($s_0$)},
    xmin=0.5, xmax=10.5,
    ymin=-7500, ymax=-2500,
    xtick={1,2,3,4,5,6,7, 8, 9, 10},
    ytick={-7000, -6000, -5000, -4000, -3000},
    legend style={
        at={(0.5,1.02)},  %
        anchor=south,      %
        legend columns=3, %
        /tikz/every even column/.append style={column sep=0.05cm}
    },
    ymajorgrids=true,
    grid style=dashed,
    cycle list name=color list,
    every axis plot/.append style={thick},
    width=10cm,
]
\addplot+[
    color=pink,
    mark=*,
    only marks,
    mark options={fill=pink, opacity=1.0, mark size=3pt}
]
coordinates {
    (1,-6950.0) (2,-6950.0) (3,-6950.0) (4,-6950.0) (5,-6950.0) (6,-6950.0) (7,-6950.0) (8,-6950.0) (9, -6950.0) (10, -6950.0)
};
\addlegendentry{Observed Path}

\addplot+[
    color=blue,
    mark=x,
    only marks,
    mark size=3pt
    ]
coordinates {
     (1,-6950.0) (2,-6950.0) (3,-6950.0) (4,-6950.0) (5,-6950.0) (6,-6950.0) (7,-6950.0) (8,-6950.0) (9, -6950.0) (10, -6950.0)
};
\addlegendentry{K=1 to K=8}

\addplot+[
    color=red,
    mark=square*,
    only marks,
    ]
coordinates {
    (1,-4819.926735468716) (2,-4188.573056914942) (3,-3818.0503867238476) (4,-3528.7296775029913) (5,-3280.1512891528228) (6,-3100.7475158150282) (7, -3014.751869563713) (8,-3014.7504863856334) (9,-3014.7504863856334) (10,-3014.7504863856334)
};
\addlegendentry{K=9 to K=T+1}

\end{axis}
\end{tikzpicture}
        }
        \label{fig:sepsis-value-catastrophic}
    }
    \hspace{1cm} %
    \subfigure[Sepsis: $V(s_0)$ given $k$-step influence and maximum $m$ actions changed for a suboptimal path.]{%
    \raisebox{0.015cm}{\resizebox{0.45\textwidth}{!}{\begin{tikzpicture}[scale=1.0, baseline=(current bounding box.south)]
{\Large
\begin{axis}[
    xlabel={Maximum Number of Actions Changed},
    ylabel={V($s_0$)},
    xmin=0.5, xmax=10.5,
    ymin=0, ymax=1.5,
    xtick={1,2,3,4,5,6,7, 8, 9, 10},
    ytick={0.0, 0.2, 0.4, 0.6, 0.8, 1.0, 1.2, 1.4},
    legend style={
        at={(0.5,1.10)},  %
        anchor=south,      %
        legend columns=3, %
        /tikz/every even column/.append style={column sep=0.05cm}
    },
    ymajorgrids=true,
    grid style=dashed,
    cycle list name=color list,
    every axis plot/.append style={thick},
    height=7cm,
    width=14cm,
]

\addplot+[
    color=pink,
    mark=*,
    only marks,
    mark options={fill=pink, opacity=1.0, mark size=4pt}
]
coordinates {
    (1,0.0) (2,0.0) (3,0.0) (4,0.0) (5,0.0) (6,0.0) (7,0.0) (8,0.0) (9, 0.0) (10, 0.0)
};
\addlegendentry{Observed Path}

\addplot+[
    color=blue,
    mark=x,
    only marks,
    mark size=4pt
    ]
coordinates {
     (1,0.0) (2,0.0) (3,0.0) (4,0.0) (5,0.0) (6,0.0) (7,0.0) (8,0.0) (9, 0.0) (10, 0.0)
};
\addlegendentry{K=1 to K=2}

\addplot+[
    color=red,
    mark=square*,
    only marks,
    mark size=2.5pt
    ]
coordinates {
    (1,1.4) (2,1.4) (3,1.4) (4,1.4) (5,1.4) (6,1.4) (7,1.4) (8,1.4) (9,1.4) (10,1.4)
};
\addlegendentry{K=3 to K=T+1}

\end{axis}}
\end{tikzpicture}}
        \label{fig:sepsis-value-suboptimal}
    }
    }
    \caption{Sepsis MDP analysis results}
\end{figure}

In the model, the patient dies if three or more vital signs go out of range. This means that Sepsis trajectories depend heavily on the first few actions to reduce the probability that treatment will lead to one of these terminal states. However, if the observed path is a catastrophic one, as in Figure \ref{fig:sepsis-value-catastrophic}, changing these actions early on will often lead far away from the observed path, and so we need a low influence (i.e., high values of $k$) to obtain reasonably better outcomes. As shown in Figure \ref{fig:sepsis-value-suboptimal}, for suboptimal paths (where the patient is not dead/discharged at the end of the path, but is not discharged), a low value of $k$ ($k \geq 3$) is sufficient to recover the optimal counterfactual path. However, only a small improvement on the observed path is possible.

\subsection{Reduction in MDP Size}
\begin{table}[h]
\centering
\resizebox{0.55\textwidth}{!}{%
{\Large
    \begin{tabular}{|c|c|c|c|c|c|c|c|c|c|c|c|c|c|c|}
        \hline
        $k$ & $1$ & $3$ & $6$ & $7$ & $10$ & $T+1$ & $|S|$ \\
        \hline
        Grid World & 12 & 15 & 15 & 147 & 182 & 192 & 16 \\
        \hline
        Epidemic & 8 & 43 & 157 & 210 & - & 19355 & 2541 \\
        \hline
        Sepsis & 11 & 14 & 3477 & 4055 & 5884 & 6996 & 1440 \\
        \hline
    \end{tabular}}
    }
    \caption{Size of the state space of pruned counterfactual MDPs, given $k$-step influence. $|S|$ is the state space size of the original MDP. There is no data for the Epidemic environment for $k=10$ because the observed path has length $7$.}
    \label{tab:state_space_after_pruning}

\end{table}

An additional benefit of our $k$-step influence approach is that the state space of the pruned counterfactual MDP can be significantly reduced. State space sizes for our models (for selected values of $k$) are given in Table \ref{tab:state_space_after_pruning}. For $k=1$, the MDP is restricted to just the states in the corresponding observed paths. 
We can also see that the state space of the pruned counterfactual MDPs for $k < T+1$ are significantly smaller (for the Epidemic and Sepsis models) than the state space of the entire counterfactual MDP ($k=T+1$), meaning that value iteration is much more efficient. Full results for all the environments are given in Appendix~\ref{app:state spaces}, and execution times are discussed in Appendix~\ref{app:execution times}.

\section{Conclusion}
In this work, we addressed a significant yet neglected issue in counterfactual inference for Markov Decision Processes (MDPs): as counterfactual states and actions progressively diverge from the observed ones over time, the observation may no longer influence the counterfactual world, and the resulting explanation will no longer be tailored to the individual observation. To tackle this issue, we introduced a formal methodology to quantify the influence of the observed path $\tau$ on a counterfactual path $\tau'$, and devised an algorithm to generate optimal counterfactual explanations and policies while satisfying predetermined influence constraints. Our experiments reveal that while there exists a trade-off between influence and policy optimality, it is often possible to derive policies that are nearly optimal while still being significantly influenced by the initial observed path. The optimal degree of influence is domain-specific, but our method allows us to evaluate the trade-off between influence and optimality and make a better-informed choice on the value of the influence bound $k$. \jl{We include a detailed discussion on the choice of $k $ in Appendix \ref{app: choice of k}.}

Although this method is the first to expose and solve the problem of counterfactual influence, it relies on the availability of the system's transition probabilities. Moving forward, our goal is to develop optimal policies through a model-free approach, particularly in scenarios where the transition probabilities of the underlying MDP are unknown or uncertain.

\acks{This work was supported by UK Research and Innovation [grant number EP/W014785/2]; and UK Research and Innovation [grant number EP/S023356/1] in the UKRI Centre for Doctoral Training in Safe and Trusted Artificial Intelligence (www.safeandtrustedai.org).}

\bibliography{refs}
\newpage
\pagebreak
\appendix
\section{Simulation Evidence for Motivating Example}
\label{app: simulation evidence}
\jl{
In Figure \ref{fig:sepsis-mdp-with-trajectories}, we showed that if we allow counterfactual paths to be unconstrained, these paths may diverge from the observation such that their counterfactual transition probabilities are uninformed by the observation. Consequently, these counterfactual paths will no longer be tailored to the specific observation. Deriving optimal counterfactual policies over areas of the counterfactual MDP that are not influenced by the observation could yield policies that are not tailored to the given observation, and may actually be suboptimal for the particular observation.

This is particularly an issue when hidden subgroup differences exist within the population. In learned MDPs, some aspects of the true, underlying state could remain unobserved, which can lead to differences in the transitions taken by different subgroups. Therefore, in these environments, it is crucial that counterfactual policies are tailored to the observation, or they may be optimal for the population as a whole, but suboptimal for the particular subgroup that the observation belongs to.

We can evaluate this potential suboptimality by considering an example MDP where we have access to the transition probability and reward functions for both a fully observable and partially observable version of the MDP. This partially observable version of the MDP represents an MDP that we may be able to learn in practice (where, for example, we can only observe an incomplete set of variables in the learned MDP). Given an observed trajectory, we can learn:

\begin{enumerate}
        \item The optimal counterfactual policy across the general population (i.e., the unconstrained counterfactual policy), using the partially observable MDP.
        \item Various $k$-CF policies, again using the counterfactual partially observable MDP.
        \item The ``true" optimal counterfactual policy for the diabetic patient, using the fully observable MDP.
\end{enumerate}

We can then compare these policies by measuring the average cumulative reward they achieve over the ``true" counterfactual MDP (i.e., the fully observable counterfactual MDP). We expect the $k$-CF policies (2) to approximate the subgroup-specific policy (3) more closely than the general-population policy (1) (and therefore achieve higher average rewards than the general-population policy) because these policies will be restricted to areas of the counterfactual MDP that are more informed/influenced by the observation. 

This effect is particularly noticeable where the observation is optimal or close to optimal. This is because these observed paths typically require very few (or no) changes to improve upon the observation, hence the optimal counterfactual paths will be close to the observation in the counterfactual MDP. When we generate counterfactual policies with a small value of $k$, this will restrict the counterfactual MDP to those areas that are highly tailored to the observation, and will include these higher reward counterfactual paths.

As an example, we take the sepsis example from Figure \ref{fig:sepsis-mdp-with-trajectories}. In the Sepsis MDP \citep{oberst2019counterfactual}, the diabetic status of the patient can be explicit or hidden in the state. Given an observed trajectory of a diabetic patient, if we derive the optimal counterfactual policy for this observation over the partially observable MDP, without constraining the policy to areas of the counterfactual MDP that are sufficiently influenced by the observation, this may lead to policies that are optimal for the general population and not optimal for the observed diabetic patient. 

\begin{figure}
    \centering
    \begin{tikzpicture}
    \begin{axis}[
        xlabel={$k$},
        xmin=1,
        xmax=11,
        ylabel={Average Cumulative Reward},
        legend style={at={(0.5,1.2)}, anchor=south, cells={align=center}}, %
        grid=major
    ]

    \addplot[color=blue, mark=o, thick] coordinates {
        (1, 0.0)
        (2, 0.0)
        (3, -11.95)
        (4, -13.56)
        (5, -15.695)
        (6, -14.605)
        (7, -17.38)
        (8, -10.99)
        (9, -15.535)
        (10, -17.225)
        (11, -13.33)
    };
    \addlegendentry{Average Reward Achieved by $k$-CF Policies}

    \addplot[color=red, dashed, thick] coordinates {
        (1, 0.0)
        (11, 0.0)
    };
    \addlegendentry{Average Reward Achieved by Optimal CF Policy Across Fully-Known MDP}

    \addplot[color=darkgreen, dash dot, thick] coordinates {
        (1, -13.33)
        (11, -13.33)
    };
    \addlegendentry{Average Reward Achieved by Optimal CF Policy Across Hidden MDP}
    \end{axis}
\end{tikzpicture}
    \caption{Average cumulative reward of policies, given an observed diabetic path under the optimal policy.}
    \label{fig:sepsis-motivation-optimal}
\end{figure}

Figures \ref{fig:sepsis-motivation-optimal} and \ref{fig:sepsis-motivation-suboptimal} present the average cumulative reward obtained by these policies on two observed diabetic trajectories. These trajectories are of length $T=10$, hence $k=11$ corresponds to the entire counterfactual MDP (see Section \ref{sec: theory}). Figure \ref{fig:sepsis-motivation-optimal} compares the policies on an observed diabetic path under the optimal policy. As expected, the optimal counterfactual policy over the fully observable MDP is the same as the observed policy, as are the $k=1$ and $k=2$ policies, which are derived from areas of the counterfactual MDP that are greatly influenced by the observation. However, the average cumulative reward achieved by the optimal counterfactual policy across the partially observable MDP is much lower, as this is optimal across the general population rather than for diabetic patients. We also see a decline in performance for $k$-CF policies where $k\geq3$, as these policies are derived over areas of the counterfactual MDP that are less influenced by (and therefore less tailored to) the observation.

\begin{figure}
    \centering
    \begin{tikzpicture}
    \begin{axis}[
        xlabel={$k$ values},
        ylabel={Reward},
        xmin=1, xmax=11, %
        legend style={at={(0.5,1.2)}, anchor=south, cells={align=center}}, %
        grid=major
    ]

    \addplot[color=blue, mark=o, thick] coordinates {
        (1, -1350.0)
        (2, -400.0)
        (3, -400.0)
        (4, -400.0)
        (5, -400.0)
        (6, -400.0)
        (7, -400.0)
        (8, -400.0)
        (9, -400.0)
        (10, -513.07)
        (11, -518.313)
    };
    \addlegendentry{Average Reward Achieved by $k$-CF Policies}

    \addplot[color=red, dashed, thick] coordinates {
        (1, -400.0)
        (11, -400.0)
    };
    \addlegendentry{Average Reward Achieved by Optimal CF Policy Across Fully-Known MDP}

    \addplot[color=darkgreen, dash dot, thick] coordinates {
        (1, -518.313)
        (11, -518.313)
    };
    \addlegendentry{Average Reward Achieved by Optimal CF Policy Across Hidden MDP}

    \end{axis}
\end{tikzpicture}
    \caption{Average cumulative reward of policies, given an observed diabetic path under a suboptimal policy.}
    \label{fig:sepsis-motivation-suboptimal}
\end{figure}

Figure \ref{fig:sepsis-motivation-suboptimal} compares the policies on an observed diabetic path under a suboptimal policy (the optimal policy with some randomly chosen actions). As expected, we see that the average cumulative reward achieved by the optimal counterfactual policy over the fully observable MDP is again higher than that of the partially observable MDP. We also see that there is a decline in performance for counterfactual policies with higher values of $k$ (in this case $k\geq10$), as these are learnt over areas of the counterfactual MDP that are not very influenced by the observation, and therefore are closer to the general population counterfactual policy rather than the ``true" diabetic counterfactual policy.
}

\section{Related Work}
\label{app:related work}
To the best of our knowledge, there is no other work that directly aims to address the problem of counterfactual influence caused by the divergence of the counterfactual path from the observed one. Nevertheless, there is growing field of work focusing on the intersection between causality and various other domains, including reinforcement learning and planning. 

\paragraph{Causal RL}
Causality can often improve the performance of RL algorithms~\citep{zeng2023survey, scholkopf2021toward}, especially when data is scarce, or where exploration may be dangerous or infeasible \citep{lu2020sample}. In such cases, counterfactual reasoning can be used to augment datasets with counterfactual data, improving the efficiency and performance of RL algorithms \citep{lu2020sample, buesing2018woulda}, to generate counterfactual paths as a causal explanation for how an observed policy could be improved \citep{oberst2019counterfactual, tsirtsis2021counterfactual, NEURIPS2023_09ae6bea, gajcin2024acter}, or to measure the influence of an individual agent's action/treatment decision on the outcome in a multi-agent setting \citep{triantafyllou2024agentspecific} (a different notion of influence to the notion of counterfactual influence in this paper).

Causal reasoning is also useful at training-time: if an agent can perform informative interventions to learn the causal structure of its environment, it would enable performing more structured exploration when learning optimal policies \citep{dasgupta2019causal}. In addition to MDPs, causal RL has also been successfully applied to 
Multi-Arm Bandit problems \citep{lattimore2019causalbandits, madhavan2021intervention} and Dynamic Treatment Regimes \citep{zhang2019dtr}.

\paragraph{Planning} 
Our work is related to the field of explainable planning~\citep{fox2017explainable} and in particular contrastive explanations~\citep{borgo2018towards}, which focuses on offering explanations about alternative sequence of actions in classic planning scenarios.  \citet{krarup2021contrastive} propose a method to restrict the model by implementing constraints based on user questions, thereby providing structured explanations for the planning procedure as a negotiation.  \citet{stein2021generating} extends this to explanations for plans in partially-revealed environments. However, it should be noted that these counterfactual explanations do not adjust the transition probabilities based on the observed path in the counterfactual world.

\paragraph{Offline RL}
In offline RL, the objective is to find an optimal policy maximising the expected return using a fixed dataset of observed trajectories \citep{uehara2022review}. However, this can be challenging due to the problem of \textit{distribution shift}, where there is a mismatch between the distribution of trajectories in the dataset, and distribution of the trajectories that would be generated by the learned policy \citep{pmlr-v139-jin21e}. This often leads to overestimation of the value function for out-of-distribution actions \citep{uehara2022review}. Recent work tackles distribution shift by promoting proximity between the learned policy and the behaviour policy. This is achieved through regularising the learned policy to avoid states and actions that appear less frequently \citep{pmlr-v97-fujimoto19a}, or using pessimistic value-based approaches, which apply a penalty to the value function on these states and actions \citep{yu2020mopo, kidambi2020}. Although our work solves a different problem, our notion of influence is similar to these methods as it can be seen as form of regularisation constraint.

\section{Discussion on Notion of Influence}
\label{app: probabilistic influence}
In our work, we formulated our notion of influence from a structural perspective, in terms of the supports of the state-action pairs. By Proposition \ref{prop: no influence disjoint supports}, we know this is an efficient and precise condition for influence. However, one could argue that this notion is too restrictive as this exerts hard constraints for pruning. Alternatively, we could have formulated $1$-step influence from a probabilistic perspective. But, any notion would have to be in terms of the interventional probability distributions alone: any notion using the counterfactual probabilities (e.g., the statistical distance between the nominal and counterfactual distributions) would not be exact, due to sampling variability in the counterfactual probabilities from the sampled Gumbel values. For example, we could measure $1$-step influence as:

\[
\dfrac{\sum_{s' \in \mathcal{S}}|P(s' \mid s_t, a_t) - P(s' \mid s, a)|}{2} = \begin{cases}
    1 \text{ if the distributions have disjoint support}\\
    <1 \text{ if the distributions have overlapping support}
\end{cases}
\]

and specify some $\epsilon$ as the maximum of this sum. Similar to our notion of $k$-step influence, we may want to extend this to paths, e.g., ensuring the total statistical distance is less than some value. We could also consider the reward: how should we find an appropriate $\epsilon$ that balances a path's total statistical distance with the total reward of that path, e.g., how should we choose parameters $\alpha$ and $\beta$ below:

\[
\alpha\dfrac{\sum_{s' \in \mathcal{S}}|P(s' \mid s_t, a_t) - P(s' \mid s, a)|}{2} - \beta \sum_{s' \in \mathcal{S}}r(s, a, s') - r(s_t, a_t, s')
\]

This is a design choice that depends on how safety-critical the domain is: ensuring the counterfactual paths are sufficiently informed vs. optimising the total reward. We chose to use our structural notion in this paper, as it is a simple and exact notion to define influence. But, in future work, it would be interesting to compare these two notions of influence, to see in what situations these two notions differ.

One simple MDP example that illustrates the differences between these notions is given in Figure \ref{fig: prob influence example}. The observed path is given in red, and the transitions that would be contained in the influence-constrained counterfactual MDP (under our structural notion of influence for $k=2$) are represented by the solid arrows. The influence-constrained counterfactual MDP is quite restrictive, largely due to the transition from $s_3 \rightarrow s_7$ which deviates far from the observed path. Under a probabilistic notion of influence, we might instead choose to ignore that this transition deviates far from the observed path because it has such low probability ($p=0.01$), and consider all of the counterfactual MDP. However, if the nominal probability of the transition from $s_3 \rightarrow s_7$ was much higher (e.g., $p>0.1$), we may want the probabilistic notion of influence to remove this path, unless, for example, the reward for reaching state $s_{11}$ was very high. The trade-off between influence and reward is a design choice and is domain-dependent, so any notion of influence should consider this: our structural notion of influence has the influence bound $k$ that can be changed to loosen the restriction on influence, and achieve higher rewards.

\begin{figure}[h]
\centering
\resizebox{0.65\textwidth}{!}{
\begin{circuitikz}
    \tikzstyle{every node}=[font=\LARGE]
    \draw [ color={rgb,255:red,255; green,0; blue,0} ] (7.5,20) circle (1.25cm) node {\LARGE $s_0$} ;
    \draw  (13.75,20) circle (1.25cm) node {\LARGE $s_1$} ;
    \draw  (10.75,16.5) circle (1.25cm) node {\LARGE $s_3$} ;
    \draw [ color={rgb,255:red,255; green,0; blue,0} ] (4.25,16.5) circle (1.25cm) node {\LARGE $s_2$} ;
    \draw  (16.75,16.5) circle (1.25cm) node {\LARGE $s_4$} ;
    \draw  (10.75,12.75) circle (1.25cm) node {\LARGE $s_6$} ;
    \draw [ color={rgb,255:red,255; green,0; blue,0} ] (4.25,12.75) circle (1.25cm) node {\LARGE $s_5$} ;
    \draw  (16.75,12.75) circle (1.25cm) node {\LARGE $s_7$} ;
    \draw [ color={rgb,255:red,255; green,0; blue,0} ] (7.75,9) circle (1.25cm) node {\LARGE $s_9$} ;
    \draw  (1.25,9) circle (1.25cm) node {\LARGE $s_8$} ;
    \draw  (13.75,9) circle (1.25cm) node {\LARGE $s_{10}$} ;
    \draw  (20,9) circle (1.25cm) node {\LARGE $s_{11}$} ;
    \draw [ color={rgb,255:red,255; green,0; blue,0}, ->, >=Stealth] (7.5,18.75) -- (5.25,17.25);
    \draw [ color={rgb,255:red,255; green,0; blue,0}, ->, >=Stealth] (4.25,15.25) -- (4.25,14);
    \draw [->, >=Stealth] (4.25,11.5) -- (2.25,9.75);
    \draw [ color={rgb,255:red,255; green,0; blue,0}, ->, >=Stealth] (4.25,11.5) -- (6.5,9.75);
    \draw [->, >=Stealth, dashed] (10.75,15.25) -- (5.5,13.5)node[pos=0.5, fill=white]{0.9};
    \draw [->, >=Stealth, dashed] (10.75,15.25) -- (15.5,13.75)node[pos=0.5, fill=white]{0.01};
    \draw [->, >=Stealth, dashed] (13.75,18.75) -- (11.75,17.25);
    \draw [->, >=Stealth, dashed] (13.75,18.75) -- (15.75,17.25);
    \draw [->, >=Stealth, dashed] (16.75,15.25) -- (16.75,14);
    \draw [->, >=Stealth, dashed] (10.75,11.5) -- (9,9.75);
    \draw [->, >=Stealth, dashed] (10.75,11.5) -- (12.5,9.75);
    \draw [->, >=Stealth, dashed] (16.75,11.5) -- (18.75,9.75);
    \draw [->, >=Stealth, dashed] (10.75,15.25) -- (10.75,14)node[pos=0.5, fill=white]{0.09};
    \draw [->, >=Stealth, dashed] (7.5,18.75) -- (9.75,17.25);
    \end{circuitikz}
}
\caption{Simple MDP example to illustrate differences between structural and probabilistic notions of influence. The observed path is given in red.}
\label{fig: prob influence example}
\end{figure}

\section{Example of Influence-Constrained Counterfactual MDP}
Take the counterfactual MDP example from Figures~\ref{fig:one-step} and ~\ref{fig:two-step}. This has two possible actions, $a_0$ and $a_1$, at states $s_0$ and $s_3$, and only action $a_0$ at the rest of the states. The full transition table for the nominal MDP is given in Table \ref{tab: transition table}. The support of each state-action pair is the set of states that can be reached (with non-zero probability) from the state-action pair, and is given in Table \ref{tab: supports table}.

\begin{table}[h]
\centering
\resizebox{0.5\textwidth}{!}{ %
    \begin{tabular}{|c|c|c|c|}
    \hline
    \textbf{State} & \textbf{Action} & \textbf{Next State} & \textbf{Transition Probability} \\ \hline
    $s_0$          & $a_0$           & $s_2$               & 0.5                             \\ \hline
    $s_0$          & $a_0$           & $s_3$               & 0.5                             \\ \hline
    $s_0$          & $a_1$           & $s_1$               & 1.0                             \\ \hline
    $s_1$          & $a_0$           & $s_4$               & 1.0                             \\ \hline
    $s_2$          & $a_0$           & $s_5$               & 1.0                             \\ \hline
    $s_3$          & $a_0$           & $s_5$               & 1.0                             \\ \hline
    $s_3$          & $a_1$           & $s_6$               & 1.0                             \\ \hline
    $s_4$          & $a_0$           & $s_7$               & 1.0                             \\ \hline
    $s_5$          & $a_0$           & $s_7$               & 1.0                             \\ \hline
    $s_6$          & $a_0$           & $s_7$               & 1.0                             \\ \hline
    \end{tabular}
}
\caption{Nominal MDP transition table}
\label{tab: transition table}
\end{table}

\begin{table}[h]
\centering
\resizebox{0.25\textwidth}{!}{ %
    \begin{tabular}{|c|c|c|}
    \hline
    \textbf{State} & \textbf{Action} & \textbf{Support} \\ \hline
    $s_0$          & $a_0$           & $\{s_2, s_3\}$     \\ \hline
    $s_0$          & $a_0$           & $\{s_3\}$            \\ \hline
    $s_0$          & $a_1$           & $\{s_1\}$            \\ \hline
    $s_1$          & $a_0$           & $\{s_4\}$            \\ \hline
    $s_2$          & $a_0$           & $\{s_5\}$            \\ \hline
    $s_3$          & $a_0$           & $\{s_5\}$            \\ \hline
    $s_3$          & $a_1$           & $\{s_6\}$            \\ \hline
    $s_4$          & $a_0$           & $\{s_7\}$            \\ \hline
    $s_5$          & $a_0$           & $\{s_7\}$            \\ \hline
    $s_6$          & $a_0$           & $\{s_7\}$            \\ \hline
    \end{tabular}
}
\caption{Supports of state-action pairs in the nominal MDP}
\label{tab: supports table}
\end{table}

Given the observed path $\tau = [(s_0, a_0), (s_2, a_0), (s_5, a_0), (s_7, a_0)]$, we can now find the influence-constrained counterfactual MDP given $k=1$ and $k=2$. When $k=1$ (Figure~\ref{fig:one-step-app}), $(s_1, a_0)$ and $(s_3, a_1)$ are not influenced at $t=1$, because they have disjoint supports ($\{s_4\}$ and $\{s_6\}$ respectively) with the observed state-action pair $(s_2, a_0)$ (whose support is $\{s_5\}$). For the opposite reason, $(s_2, a_0)$ and $(s_3, a_0)$ are influenced at $t=1$, as the supports of $(s_2, a_0)$ and $(s_3, a_0)$ are both $\{s_5\}$.

$(s_4, a_0)$, $(s_5, a_0)$ and $(s_6, a_0)$ are all influenced at $t=2$, as their supports overlap with the observed pair $(s_5, a_0)$ (in fact, all of their supports are exactly $\{s_7\})$. However, even though $(s_4, a_0)$ and $(s_6, a_0)$ are influenced, they cannot be reached from any influenced state-action pairs, so are also removed from the influence-constrained counterfactual MDP: we say they are ``influenced but unreachable''.

Figure \ref{fig:two-step-app} depicts the case of 2-step influence. We note that $(s_6, a_0)$ is now reachable, because although the support of $(s_3, a_1)$ is disjoint from the support of the observed state-action pair $(s_2, a_0)$ (and so is not 1-step influenced), all transitions leading to $(s_3, a_1)$ and from the states that $(s_3, a_1)$ can reach are influenced, so $(s_3, a_1)$ is influenced at $t=1$ with 2-step influence. However, even though $(s_1, a_0)$ now becomes influenced at $t=1$, it cannot be reached by any influenced state-action pair, so $(s_1, a_0)$ and $(s_4, a_0)$ are influenced but unreachable.

\begin{figure}[ht]
    \centering
    \resizebox{0.45\textwidth}{!}{ %
    \begin{tikzpicture}
        \matrix [draw, below=1cm of current bounding box.south] {
            \node [main node, fill=black!50,label=right:Observed] (l1) {}; \pgfmatrixnextcell
            \node [main node, fill=black!25,label=right:Influenced and reachable] (l2) {}; \\
            \node [main node,label=right:Influenced but unreachable] (l3) {}; \pgfmatrixnextcell
            \node [main node, gray,label=right:Not $k$-step influenced] (l4) {}; \\
        };
    \end{tikzpicture}
    }

    \subfigure[$k=1$]{%
         \resizebox{0.25\textwidth}{!}{\begin{tikzpicture}[->,>=stealth,node distance=1.2cm,thick,main node/.style={circle,draw}, ]

      \node[main node, fill=black!50] (s0) {$s_0$};
      \node[main node, fill=black!50] (s2) [right of=s0] {$s_2$};
      \node[main node, gray] (s1) [below of=s2] {$s_1$};
      \node[main node, fill=black!25] (s3) [above of=s2] {$s_3$};
      \node [above] at (s3.north) (t1) {t=1};
      \node [] (t0) [left of=t1] {t=0}; %
      \node[main node] (s4) [right of=s1] {$s_4$};
      \node[main node, fill=black!50] (s5) [right of=s2] {$s_5$};
      \node[main node] (s6) [right of=s3] {$s_6$};
      \node [above] at (s6.north) (t2) {t=2};
      \node[main node, fill=black!50] (s7) [right of=s5] {$s_7$};
      \node [] (t3) [right of=t2] {t=3}; %
      \draw[gray] (s0) -- (s1) node[midway,right] {$a_1$};
      \draw[ultra thick] (s0) -- (s2) node[midway, above] {$a_0$};
      \draw (s0) -- (s3) node[midway,right] {$a_0$};
      \draw[gray] (s1) -- (s4) node[midway,above] {$a_0$};
      \draw[ultra thick] (s2) -- (s5) node[midway,above] {$a_0$};
      \draw[gray] (s3) -- (s6) node[midway, above] {$a_1$};
      \draw (s3) -- (s5) node[midway, right] {$a_0$};
      \draw (s4) -- (s7) node[midway, right] {$a_0$};
      \draw[ultra thick] (s5) -- (s7) node[midway, above] {$a_0$};
      \draw (s6) -- (s7) node[midway, right] {$a_0$};

    \end{tikzpicture}}
        \label{fig:one-step-app}
    }
    \hspace{0.05\textwidth} %
    \subfigure[$k=2$]{%
         \resizebox{0.25\textwidth}{!}{\begin{tikzpicture}[->,>=stealth,node distance=1.2cm,thick,main node/.style={circle,draw}]

      \node[main node, fill=black!50] (s0) {$s_0$};
      \node[main node, fill=black!50] (s2) [right of=s0] {$s_2$};
      \node[main node] (s1) [below of=s2] {$s_1$};
      \node[main node, fill=black!25] (s3) [above of=s2] {$s_3$};
      \node [above] at (s3.north) (t1) {t=1};
      \node [] (t0) [left of=t1] {t=0}; %

      \node[main node] (s4) [right of=s1] {$s_4$};
      \node[main node, fill=black!50] (s5) [right of=s2] {$s_5$};
      \node[main node, fill=black!25] (s6) [right of=s3] {$s_6$};
      \node [above] at (s6.north) {t=2};
      \node[main node, fill=black!50] (s7) [right of=s5] {$s_7$};
      \node [] (t3) [right of=t2] {t=3}; %
      \draw[gray] (s0) -- (s1) node[midway,right] {$a_1$};
      \draw[ultra thick] (s0) -- (s2) node[midway, above] {$a_0$};
      \draw (s0) -- (s3) node[midway,right] {$a_0$};
      \draw (s1) -- (s4) node[midway,above] {$a_0$};
      \draw[ultra thick] (s2) -- (s5) node[midway,above] {$a_0$};
      \draw (s3) -- (s6) node[midway, above] {$a_1$};
      \draw (s3) -- (s5) node[midway, right] {$a_0$};
      \draw (s4) -- (s7) node[midway, right] {$a_0$};
      \draw[ultra thick] (s5) -- (s7) node[midway, above] {$a_0$};
      \draw (s6) -- (s7) node[midway, right] {$a_0$};

\end{tikzpicture}}
        \label{fig:two-step-app}
    }

    \caption{Example counterfactual MDP given $k$-step influence. State-action pairs may or may not be influenced by the observed path, and states may or may not be reachable from other influenced state-action pairs.}
\end{figure}
\label{app: influence constrained mdps}

\newpage
\section{Algorithm for Constructing Influence-Constrained Counterfactual MDP}\label{app:pseudocode}
\begin{algorithm}
\caption{Find Optimal \((k, m)\)-CF Policy for a given MDP}
\begin{algorithmic}[1]

\STATE \textbf{Input:} MDP transition probabilities $P$, observed path $\tau$, counterfactual transition probabilities $P_{\mathcal{P},t,\tau}$, $k$, $m$
\STATE \textbf{Output:} Optimal $(k, m)$-CF policy $\pi^*$

\STATE $S^\tau \gets \emptyset$ \COMMENT{Initialise set of all states in support of each observed $(s_t, a_t)$}

\FOR{each state-action pair $(s_t, a_t)$ in the observed path}
    \STATE $S^\tau_t \gets$ all states in the support of $P(\cdot \mid s_t, a_t)$:  $\{s'\mid P(s' \mid s_t, a_t)>0\}$
    \STATE $S^\tau \gets S^\tau \cup S^\tau_t$
\ENDFOR

\STATE $S^{\tau, k} \gets \emptyset$ \COMMENT{Initialise set of all states which are $k$-step influenced}

\FOR{each state $s$ in $S^\tau$} 
    \STATE $S^{\tau, k} \gets S^{\tau, k} \cup \text{ReverseBFS}(s, k)$ \COMMENT{Reverse BFS with depth $k$}
\ENDFOR

\FOR{$t$ in range$(0, T-k+1)$} 
    \FOR{each $s$, $a$, $s'$} 
        \IF{$P(s' \mid s, a) > 0$ and $s' \not \in S^{\tau, k}$}
            \STATE $P_{\mathcal{P},t,\tau}(\cdot \mid s, a)=0$ \COMMENT{Prune non-influenced transitions}
        \ENDIF
    \ENDFOR
\ENDFOR

\STATE Further prune MDP to remove transitions leading to unreachable states and states with no outgoing edges

\STATE Compute the optimal $(k, m)$-CF policy using dynamic programming, while restricting action choices to transitions in the influence-constrained counterfactual MDP

\RETURN \COMMENT{Optimal $(k, m)$-CF policy}

\end{algorithmic}
\end{algorithm}

\section{Epidemic MDP}\label{app:epidemic_mdp}
The Epidemic MDP models how infection spreads through a given population $P$ of vaccinated and unvaccinated individuals. The MDP uses a hypergeometric distribution to model how many susceptible individuals become infected at each step. Each vaccination decreases the count $V$ and removes the vaccinated individual from the population (i.e., no re-infection is possible). The reward for each transition $(s, a, s')$ is given by the negative of the number of infected individuals in $s$, $-I$.

The MDP can be described as follows.

\paragraph{State Space} The state space consists of a tuple $(S, I, V)$ where:
\begin{itemize}
    \item $S$: number of individuals susceptible to the disease (ranging from 0 to $P$).
    \item {I}: number of individuals infected with the disease (ranging from 0 to $P$).
    \item {V}: number of vaccines available (ranging from 0 to $2 \times P$).
\end{itemize}

\paragraph{Initial State} The initial state $(S_0, I_0, V_0)$ consists of:
\begin{itemize}
    \item $S_0 = P - I_0$ (initially the entire population is unvaccinated).
    \item $I_0$ is chosen arbitrarily or can be taken from any chosen distribution. In our experiments, we set $I_0 = 1$.
    \item $V_0 = 2 \times P$.
\end{itemize}

\paragraph{Action Space} There are three possible actions at each time step:
\begin{itemize}
    \item $V_I$: vaccinate an infected individual.
    \item $V_S$: vaccinate a susceptible individual.
    \item $\mathit{Nil}$: do nothing.
\end{itemize}

\paragraph{Transition Probabilities} The transition probabilities are defined as follows. We assume that at each time step, individuals in $S_t$ can be infected following a hypergeometric model, i.e., a binomial without replacement.

\begin{itemize}
    \item For the action $\mathit{Nil}$:
    \begin{itemize}
        \item $P(S_{t+1}, I_{t+1}, V_{t+1} \mid S_t, I_t, V_t, NIL) = 0$ if $V_{t+1} \neq V_t$.
        \item For $k \leq S_t$, $P(S_{t+1}-k, I_{t+1}+k, V_{t+1} \mid S_t, I_t, V_t, \text{NIL}) = \text{hypergeom}(M, n, N).\text{pmf}(k)$ if $V_{t+1} = V_t$, where $M = S_t + I_t$, $n = \min(S_t, I_t)$, $N = S_t$.
    \end{itemize}

    \item For the action $V_I$:
    \begin{itemize}
        \item $P(S_{t+1}, I_{t+1}, V_{t+1} \mid S_t, I_t, V_t, V_I) = 0$ if $V_{t+1} \neq V_t-1$.
        \item For $k \leq S_t$, $P(S_{t+1}-k, I_{t+1}-1+k, V_{t+1} \mid S_t, I_t, V_t, V_I) = \text{hypergeom}(M, n, N).\text{pmf}(k)$ if $V_{t+1} = V_t-1$, where $M = S_t + I_t - 1$, $n = \min(S_t, I_t-1)$, $N = S_t$.
    \end{itemize}

    \item For the action $V_S$:
    \begin{itemize}
        \item $P(S_{t+1}, I_{t+1}, V_{t+1} \mid S_t, I_t, V_t, V_S) = 0$ if $V_{t+1} \neq V_t-1$.
        \item For $k \leq S_t-1$, $P(S_{t+1}-k-1, I_{t+1}+k, V_{t+1} \mid S_t, I_t, V_t, V_S) = \text{hypergeom}(M, n, N).\text{pmf}(k)$ if $V_{t+1} = V_t-1$, where $M = S_t + I_t - 1$, $n = \min(S_t-1, I_t)$, $N = S_t-1$.
    \end{itemize}
\end{itemize}

\paragraph{Rewards} The reward function at each time step $t$ is defined as the negative of the number of infected individuals, $R_t = -I_t$.

\newpage
\section{Size of State Space of Pruned Counterfactual MDPs, Given \textit{k}-step Influence}
\label{app:state spaces}

\begin{table}[h]
    \centering
    \caption{Grid World: Size of the State Space of Pruned Counterfactual MDP, Given $k$-step influence}
    \begin{tabular}{|c|c|c|c|c|c|c|c|c|c|c|c|c|c|c|}
        \hline
        $k$ & 1 & 2 & 3 & 4 & 5 & 6 & 7 & 8 & 9 & 10 & 11 & T+1 & |S|\\
        \hline
        State Space & 12 & 12 & 15 & 15 & 15 & 15 & 147 & 161 & 173 & 182 & 188 & 192 & 16 \\
        \hline
    \end{tabular}
    \label{tab:gridworld_state_space_after_pruning}
\end{table}

\begin{table}[h]
\centering
    \caption{Epidemic: Size of the State Space of Pruned Counterfactual MDP, Given $k$-step influence}
    \begin{tabular}{|c|c|c|c|c|c|c|c|c|c|c|c|c|c|c|c|}
        \hline
        $k$ & 1 & 2 & 3 & 4 & 5 & 6 & 7 & T+1 &|S| \\
        \hline
        State Space & 8 & 32 & 43 & 59 & 91 & 157 & 210 & 19355 & 2541 \\
        \hline
    \end{tabular}

    \label{tab:epidemic_state_space_after_pruning}
\end{table}

\begin{table}[h]
    \caption{Sepsis: Size of the State Space of Pruned Counterfactual MDP, for the Catastrophic Path, Given $k$-step influence}
    \centering
    \begin{tabular}{|c|c|c|c|c|c|c|c|c|c|c|c|c|c|c|}
        \hline
        $k$ & 1 & 2 & 3 & 4 & 5 & 6 & 7 & 8 & 9 & 10 & T+1 & |S| \\
        \hline
        State Space & 11 & 14 & 14 & 2123 & 2896 & 3477 & 4055 & 4647 & 5291 & 5884 & 6996 & 1440 \\
        \hline
    \end{tabular}
    \label{tab:sepsis_state_space_after_pruning}
\end{table}

\section{Training Details}\label{app:execution times}
\label{Training Details}
Our algorithm was implemented in Python 3.10 and executed on a 128-core machine with an Intel Xeon CPU and 512 GB RAM, but only 32 threads were required to calculate the counterfactual transition probabilities, which was the only parallelised part of the algorithm.

The Grid World case study was relatively quick as this has a relatively small state space: for a fixed choice of $(k, m)$, deriving the (pruned) counterfactual MDP and computing the optimal policy runs in the order of minutes. The Epidemic and Sepsis case studies have larger state spaces, and so it took several hours to derive the counterfactual MDP and run policy iteration for every combination of $(k, m)$.

\section{Discussion on Choice of $k$}
\label{app: choice of k}
\jl{The parameter $k$ sets the level of influence that we consider `sufficient’ for our counterfactual paths to be informed by the observation. The choice of $k$ is domain-dependent, and there may not necessarily be a ``correct” value of $k$. Instead, we consider $k$ to be a design choice. For example, in healthcare and other safety-critical domains, it is desirable that any counterfactual path (which will be used as a counterfactual explanation for how the current treatment policy could be improved) is well informed by the observation: this would naturally lead to choosing low $k$ values. However, for less safety-critical domains, we are more concerned about optimising the reward, at the risk of doing so over non-influenced paths (i.e., paths that are not tailored to the observation). In such cases, we may want to choose higher $k$ values, (e.g., the smallest $k$ s.t. the counterfactual reward meets some threshold).

In some domains, it may be possible to select an appropriate $k$ for a particular observation. For example, in the Sepsis MDP experiment, you can identify counterfactual policies and counterfactual paths for different values of $k$, and a domain expert (e.g., a clinician) could assess these paths and identify whether they would be realistic or unrealistic for the observed patient, thereby allowing us to identify counterfactual policies that are tailored to the individual. For example, in our Sepsis example in Figure \ref{fig:sepsis-mdp-with-trajectories}, a clinician might be able to tell that the unconstrained counterfactual path is unrealistic for the observed trajectory of the diabetic patient, e.g., because the patient’s blood sugar levels in the unconstrained counterfactual path look ``too stable” for a diabetic patient with Sepsis. By comparing counterfactual trajectories generated at different levels of $k$, the clinician may be able to set a maximum $k$ below which the counterfactual paths appear reliable, based on their expert knowledge and experience.

However, in other domains, it may be more challenging to evaluate whether a counterfactual path remains tailored to the particular observation. Here, the choice of $k$ will depend on the given task. If adherence to the observation is important (e.g., personalised therapy), a small value of $k$ may be preferable to ensure that any policy changes remain informed by the original observation. On the other hand, in less safety-critical domains, we may choose a large value of $k$ to allow for larger deviations from the observation, particularly if the observed path was catastrophic, in the hopes of achieving higher rewards.
}

\end{document}